\documentclass[11pt]{article}
\usepackage[utf8]{inputenc}
\usepackage[margin=1in]{geometry}

\usepackage{amsmath,amsthm,amsfonts,amssymb,mathdots,array,mathrsfs,bm,bbm,stmaryrd,graphicx,subfigure,xcolor}
\usepackage{microtype}
\usepackage{graphicx}
\usepackage{subfigure}
\usepackage{booktabs} 
\usepackage{natbib}
\setcitestyle{open={(},close={)}}
\usepackage[colorlinks=true,citecolor=blue]{hyperref}

\usepackage{multicol}

\usepackage[ruled,vlined]{algorithm2e}

\usepackage{amsmath}
\usepackage{amsfonts}
\usepackage{amssymb}
\usepackage{mathbbol}
\usepackage{graphicx}
\usepackage{fullpage}
\usepackage{url}
\usepackage{csquotes}
\usepackage[american]{babel}
\usepackage{comment}
\usepackage{multirow}
\usepackage{url}
\usepackage{times,subfigure}

\newtheorem{theorem}{Theorem}
\newtheorem{lemma}[theorem]{Lemma}

\newtheorem{corollary}{Corollary}

\newtheorem{problem}{Problem Definition}

\usepackage{xcolor}

\begin{document}

\title{\huge Planning and Formulations in Pursuit-Evasion:  Keep-away Games and Their Strategies}

\author{\vspace{0.5in}\\\textbf{Weifu Wang} \ and  \ \textbf{Ping Li} \\\\
Cognitive Computing Lab\\
Baidu Research\\
%No. 10 Xibeiwang East Road, Beijing 100193, China\\
10900 NE 8th St. Bellevue, WA 98004, USA\\\\
  \texttt{\{harrison.wfw,  pingli98\}@gmail.com}
}

\date{\vspace{0.5in}}
\maketitle

\begin{abstract}\vspace{0.3in}

\noindent We study a pursuit-evasion problem which can be viewed as an extension of the {\em keep-away} game. In the game, pursuer(s) will attempt to intersect or catch the evader, while the evader can visit a fixed set of locations, which we denote as the {\em anchors}. These anchors may or may not be stationary. When the velocity of the pursuers is limited and considered low compared to the evaders, we are interested in whether a winning strategy exists for the pursuers or the evaders, or the game will draw. When the anchors are stationary, we show an algorithm that can help answer the above question. \\

\noindent The primary motivation for this study is to explore the boundaries between kinematic and dynamic constraints. In particular, whether the solution of the kinematic problem can be used to speed up the search for the problems with dynamic constraints and how to discretize the problem to utilize such relations best. In this work, we show that a geometric branch-and-bound type of approach can be used to solve the stationary anchor problem, and the approach and the solution can be extended to solve the dynamic problem where the pursuers have dynamic constraints, including velocity and acceleration bounds. 
\end{abstract}

\newpage
\section{Introduction}
% \vspace{-0.05in}

We study a pursuit-evasion game where the decision space for the pursuer is continuous, while the decision space for the evaders is a discrete set. We focus on the problem, which can be viewed as an extension of the {\em keep-away} game, where the pursuers (interceptors) attempt to intercept a {\em ball} (evader) moving among a set of sites. Many children and footballers play variations of this game to sharpen their reflexes and speed. We study this problem to discover winning strategies involved in such games and whether the strategies can be used to solve more complex variations of the problem. 

In the proposed game and many similar ones, local perturbations of controls or trajectories may not affect the outcomes except in some finite regions. In other words, we are trying to identify the equivalence of control synthesis for such pursuit-evasion games by analyzing the limit and the discrete versions of the problem. In a discrete setting, which is the main focus of this work, we use geometry and search algorithms to find the winning strategies. In future work, we will study whether such solutions can infer decision regions and control boundaries for complex problems.

The fundamental motivation of this study is to understand the differences between decisions and controls. In non-chaotic systems, controls often can variate by a small amount without changing the outcome, and the system would be controllable within such bounded regions. We refer to the non-similar behaviors across the boundaries of the regions that would lead to different outcomes as different {\em decisions}. 

It is often difficult to directly find the optimal controls for most systems. Nevertheless, if the decision boundaries can be identified, we only need to search and optimize within the same decision structure and compare different decision structures to find the correct outcome of the problem. We use the proposed game as an example to study when the topology matters in the game's outcome and geometry.\\

The proposed game can be briefly described as follows:
\begin{itemize}
\item There are two players, $E$ (evader) and $P$ (pursuer).
\item At every time $E$ can be either at one of finitely many states $s\in S$, or in transition between two states. Whenever $E$ is at a state $s$ it has to choose to transit to the another state $s'$, and the transition time is $T_{s, s'}$
\item $P$ travels on a metric space $X$ with bounded velocity $v$.
\item For any pair of states $s$, $s'$, there is a subset of $X$, $A_{s, s'}$, such that if $P$ is at $A_{s, s'}$ when $E$ is at $s$, then $P$ wins.
\item The strategy of $E$ at time $t$ depends on the position of $P$ on $(0, t)$. The strategy of $P$ at $t$ depends on the strategy of $E$ at $(0, t]$.
\end{itemize}
The pursuers win when they can intercept the evader in a finite time; otherwise, they lose. Even in the highly symmetrical case where the $s$ are placed on the vertices of regular polygons, finding a winning strategy is not trivial and sometimes challenges our intuition.

\newpage

\begin{figure}[t]

    \centering
    \includegraphics[width=5in]{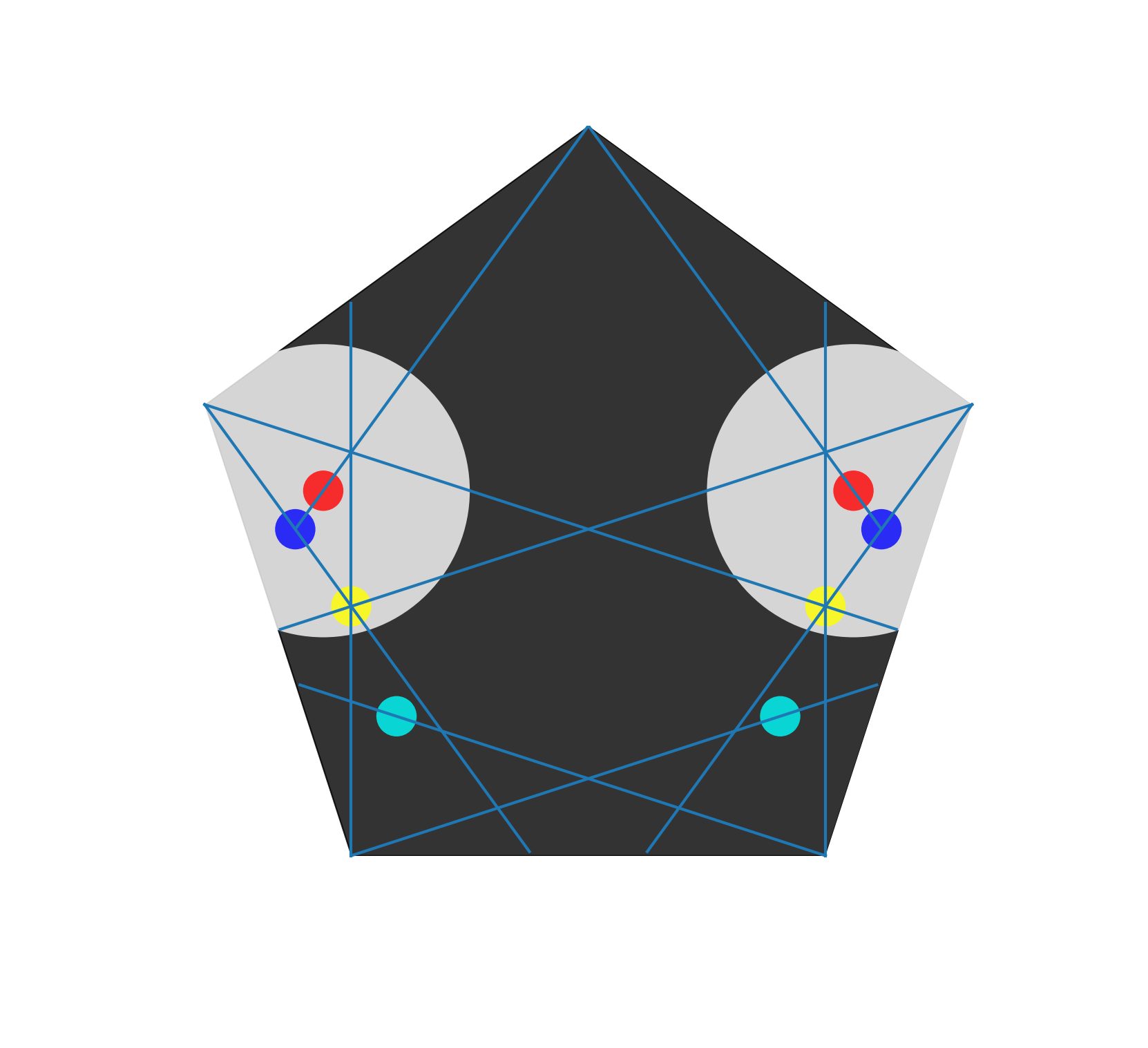}
    
\vspace{-0.6in}

    \caption{Example of winning strategy for two pursuers against five passers. Start from blue to red circles, or from blue to yellow to light blue circles, depending on where the ball is currently located when the pursuers start from the blue circles. }
    \label{fig:five_against_two}\vspace{0.2in}
\end{figure}

For example, in Figure~\ref{fig:five_against_two}, we show the winning strategy for the two pursuers against a single evader $e$, where the $s\in S$ ($|S| = 5$) forms a regular pentagon, and evader velocity $v_e$ is greater or equal to pursuer velocity $v_p$ times $\sin{\frac{\pi}{10}}$. 

The pursuer can start from the blue circles or reach blue circles from arbitrary starting positions. Once they reach the blue circles, the evader cannot travel to the bottom two states from the top three states or vice versa. Based on where the evader is, the pursuer either moves toward the red circles and eventually approaches the state on the top to intercept the evader or moves towards the yellow circles, then to the light-blue circles to intercept the evader. 

Many real-life problems share similar characteristics with the proposed keep-away game. Beyond the sporting strategies, tracking or manhunt can be viewed as intercepting the evaders with known states $S$. Problems like sensor or infrastructure placement and message passing between a set of servers share the game mechanics and the geometric properties of the proposed keep-away game. The proposed game is similar to the homicidal chauffeur game, where the pursuers have more maneuverability. Another real-life scenario can also be viewed as an extension of the proposed game: satellite communication, coverage reconfiguration, and message interception. The main extension would be that the state set $S$ is no longer in Euclidean space but on a manifold. We will study these variations in future work. 

\section{Pursuit and Evasion Games}

Pursuit and evasion games have attracted attention from the math and robotics community for decades. One of the first pursuit and evasion games was the {\em lion and man} game: a lion and a man (each viewed as a single point) in a closed disc have equal maximum speeds; can the lion catch the man? The game is almost 100 years old, and the solution is not apparent. The winning side and the winning strategy spurred debates for many years. The {\em lion and man} game has a unique character: discretization changes the game's outcome. 

If both the lion and the man move continuously in time and no bound exists for the acceleration or curvature of the path, then it was shown that the man could escape capture indefinitely, by Mathematician Abram Samoilovitch Besicovitch, as reported later in a book by~\citet{littlewood1986littlewood}. Researchers suspected the lion could be the winner by staying on the same radius vector as the man at maximum speed. This strategy only works when the man stays on a fixed radius circle, such as the boundary of the disc. Besicovitch showed that it was wrong to assume the optimal strategy for the man is to stay on the boundary of the disc but run along the direction perpendicular to the radius vector. If we discretize the game, for example, let the game plays in turns, then the game's outcome can be changed. When the environment is bounded, the lion can always catch man~\citep{bollobas2012lion}. If we discretize the locations inside the disc, the lion wins. 

There are many variations of the pursuit-evasion games either in the Lion-and-Man framework~\citep{croft1964lion,alonso1992lion,sgall2001solution, karnadI2009lion, bollobas2012lion, abrahamsen2017best,   barmak2018lion, garciac2020pride} or some altered models, such as the Homicidal Chauffeur game~\citep{isaacs1999differential}, Angle game~\citep{conway1996angel}, Princess and monster game~\citep{isaacs1999differential}. Further, \citet{isaacs1999differential} presented a differential form of the pursuit-evasion game as a simplification of the missile guidance and interception problem. Similar models were also used to track down tagged wildlife for information gathering and biological studies. Recently,~\citet{garciac2020pride} studied the game of defender against spies as a differential game.

There are further variations of pursuit-evasion games in the robotics community. For example, people have studied the pursuit-evasion games with visibility constraints~\citep{guibas1997visibility,gerkey2006visibility,bhattacharya2010existence, li2018search}, within challenging environments~\citep{alexander2006pursuit}, with dynamic constraints~\citep{weintraub2020introduction}, or as a model for multi-agent planning~\citep{discenza1981optimal,shoham2008multiagent}.

We present a zero-sum pursuit-evasion game, attempting to study the difference between discrete and continuous decisions. The game is primarily an extension of the keep-away game often practiced in football and other sports. We show that the game is not trivial, even simplified, and the optimal strategies for two sides are closely related to placement and path planning.

\newpage

\section{The Keep-away Game}

\begin{problem}
Let us consider the following game. Pick $n$ disjoint points $p_i\in\mathbb{R}^2$, each denote a state $s\in S$. Then we allow $m$ pursuers to move freely in $\mathbb{R}^2$ starting at some given initial positions, with the only constraint being velocity no larger than $v_p$. The evader must travel between states $s$ with velocity $v_e$. The pursuers win if and only if they can intercept the evader at a finite time.
\label{prob:basic}
\end{problem}

In this work, we study the following questions:
\begin{itemize}
    \item What are the values of $m$ and $n$ such that the pursuers always has a winning strategy?
    \item Given $m$ and $n$, what is the threshold for $v_p$ and $v_e$ such that either side always has a winning strategy?
    \item How to determine if there is a winning strategy and what is the strategy?
\end{itemize}

The game is zero-sum, with heterogeneous agents, like the homicidal chauffeur game~\citep{isaacs1999differential}. In addition to maneuverabilities, the decision spaces for the two sides are also different. There have been studies on the keep-away game, but most use search or learning methods to find good strategies~\citep{gerkey2006visibility,yuan2007rational,  gao2012argumentation, zhao2017local}.

% Other extensions.
The game proposed above is a simple version of what we would like to study in more depth in the future. Currently, the game involves only kinematic constraints and no dynamic constraints. This study means to gain insights into the more complex problem: the states $s$ are not stationary, and the pursuers have acceleration bounds. We hope the proposed kinematic game strategies can assist the analysis of the differential game when the players have acceleration bounds. 

\vspace{0.1in}

\noindent\textbf{Complexity.} We first show that the proposed problem is equivalent to a turn-based game, where one side has finitely many options, and the other has a continuous family of options:

\vspace{0.1in}
\begin{theorem}
Given the problem defined in Problem Definition~\ref{prob:basic}, the game has the same outcome if the pursuers and evaders make decisions simultaneously or in turns.
\end{theorem}

\begin{proof}
Let $\pi_d^P$ be the optimal strategy for turn-based play for the pursuers, and $\pi_s^P$ be the strategy for the synchronized play for the pursuers. Similarly, let $\pi_d^E$ be the optimal strategy for turn-based play for the evaders, and $\pi_s^E$ be the strategy for the synchronized play.

If $\pi_d^P \neq \pi_s^P$, then there exist a pursuers $p_i$ and a placement $C = \{c_{\cdot}\}$, so that the passing choice is different for $p_i$ if game is played in turn or in a synchronized fashion. There is no acceleration bound for pursuers, and given any interval $\Delta t$, a pursuer can reach any point in a circle of radius $v_p\cdot\Delta t$. Given the same pursuer placements, whether the pursuers can intercept the ball, i.e., the interception region, is fixed. If there is no observation delay, the interception region is the union of the circles with radius $v_c\cdot\Delta t$, where $\Delta t$ is the time needed for the evader to reach the next state $s\in S$. If there is an observation delay of $\epsilon$, the interception region is the union of the circles with radius $v_c\cdot(\Delta t - \epsilon)$. Therefore, the strategy of a particular passer is independent of the gameplay fashion, $\pi_d^P$ and $\pi_s^P$ are the same.

When we consider the pursuers, as the trajectory for the ball is a straight line, the condition for interception is entirely determined by the reachable set of the pursuers, whether there is an observation delay or not. We can conclude that the strategy for the pursuers would be the same, and the same argument holds for the evaders. 
\end{proof}

\section{Analysis of Winning Strategies}

We first show some game situations where both sides' winning strategies are straightforward.

\subsection{The case of $m\geq n-1$}

\begin{theorem}
When $m\geq n-1$, the pursuers always has a winning strategy.
\end{theorem}
\begin{proof}
In this case, $n-1$ pursuers will be enough to play {\em guards}, i.e., each approach and stay at the location of one of the state $s$ the evader must visit. When the pursuers all arrive at a unique state $s$, if the evader is at one of those $n-1$ states, then pursuers have won. If not, pursuers can move towards the remaining state and intercept the evader wherever it moves next.
\end{proof}

\subsection{A case where the passer always wins}

To describe the situation, we need the next lemma:

\vspace{0.1in}
\begin{lemma}
If the evader moves velocity $v_e$ from a location $p$ to another state $s$ very far away, the pursuer can intercept the evader if and only if the direction from $p$ to the location of the pursuer and the direction to $s$ has an angle $\theta$ with $\sin(\theta)\leq v_p/v_e$.
\end{lemma}

\begin{proof}
Suppose the pursuers, originally at point $R$, catches the ball at point $Q$, then in the triangle, $\Delta PQR$ we have the length of $RQ$ divided by the length of $PQ$ being no more than $v_p/v_e$, and this is the sine of the angle at $P$ divided by the sine of the angle at $R$, the latter is bounded by $1$.
\end{proof}

\vspace{0.1in}

Now, define a pass that can be blocked if a pursuer's direction to $p$ is no more than $\arcsin(v_p/v_e)$ to this direction. We have:

\vspace{0.1in}

\begin{theorem}\label{blocked}
 The evader has a winning strategy when the number of the pursuers is small and their velocities are very low so that they cannot block all moving directions of the evader.
\end{theorem}

\begin{proof}
The strategy of the evader is to always move in a direction that has not been blocked.
\end{proof}

\vspace{0.1in}

In the later section, with a slight abuse of terminology, we will also call a set of adjacent states $S_k$ being {\em blocked} if all the lanes for the evaders to move towards a state, not in $S_k$ can be {\em intercepted} by a pursuer before the evader reach the desired target state. 

\subsection{The case of circle}

Now consider the case when the states are placed with equidistance around a unit circle. At a given state, the evader can move towards $n-1$ different states, and the angle between adjacent moving directions is $\pi/n$. Hence Theorem~\ref{blocked} now implies:

\vspace{0.1in}

\begin{corollary}
If $v_p/v_e<\sin((\lceil {n-1\over m}\rceil-1){\pi\over 2n})$, the evader has a winning strategy.
\end{corollary}

\vspace{0.1in}
Now we consider the case $m=1$. In this case, there is a complete description of the outcome of the game because of the following Theorem:

\begin{theorem}
When $m = 1$ and $v_p/v_e \geq \sin(\pi/2-\pi/n)$, the pursuers have winning strategy.
\end{theorem}

\begin{proof}
When $n=3$ or $4$, the pursuer can go to the center of the unit circle, then it should be able to intercept the evader regardless of its moving direction. For $n\geq 5$, the pursuer can use the following strategy, inspired by the lion and man game strategy: first run towards the center. Then, if the evader moves toward a state $s$, run towards $s$; if the evader does nothing, run towards the state where the evader resides. Now the center, state where the evader resides, and pursuer all lie on the same line. Whenever the evader moves towards a state not on the line, the pursuer run perpendicular to this line till the pursuer, center, and target state lies on the same line again, then stops and waits till the evader reaches the new state. Otherwise, run towards the state the evader resides. One can see from trigonometry that this strategy can always exist, and with more and more steps, the distance between pursuer and evader converges to $0$. When the pursuer is close enough to a passer, it can now just block all the moving lanes.
\end{proof}

When there is one pursuer, and $n$ approaches $\infty$, the problem becomes increasingly analogous to the lion and man game. The difference is that the pursuer only needs to get close enough to win this problem. The pursuer can win in finitely many steps when $v_p$ is above the given threshold.

If the number of pursuers is between these two extreme values, and the velocity of the pursuer is not very large, the winning condition for the pursuers becomes non-trivial. We first present the winning conditions for five states and two pursuers in the next section and then present the analysis and the algorithm to find generalized solutions to the proposed~problem.

\section{Five States with Two Pursuers}

Consider the case of five states placed on the vertices of a regular pentagon. Label the states $s_0$ through $s_4$, with $s_0$ on the $y$ axis and the indices increment counter-clockwise. Consider two pursuers with velocity $v_p$ just sufficient to block all moving lanes against a single evader. We can compute the minimum velocity for the two pursuers to block all moving lanes~is~$v_e\sin(\frac{\pi}{10})$. 

Below is a winning strategy for the pursuers. Let the two pursuers reach the blue circles  in Figure~\ref{fig:five_against_two}  regardless of the evader strategy, and wait for the evader to leave the current state and move towards another state.  Considering symmetry, there are three possible states where the evader can start: $s_0$, $s_1$ ($s_4$), $s_2$ ($s_3$).

\vspace{0.1in}

\noindent\textbf{Evader start from $p_0$}: Evader can only goes to $s_1$ or $s_4$, as the blue circles currently blocks the lanes to $s_2$ or $s_3$; evader remains at $s_0$. If the evader moves towards $s_2$ or $s_3$, the pursuers stay at the blue circles. If the evader remains at $s_0$, the pursuers move towards the red circles and follow the line towards $s_0$. If the evader is still at $s_0$ when the pursuers reach the red circles, all lanes are now blocked, and pursuers win. If the evader moves towards $s_1$ or $s_4$, the pursuers return to blue circles.

\vspace{0.1in}

\noindent\textbf{Evader start from $s_1$}: The evader cannot move towards $s_2$ or $s_3$, or $s_4$, as these lanes are blocked. The evader can only move towards $s_0$, and the pursuers win the next round by heading towards the red circles.

\vspace{0.1in}
 %%%% HERE
\noindent\textbf{Evader start from $s_2$}: The evader cannot head towards $s_1$ or $s_4$. If the evader heads towards $s_0$, the pursuers can win in the next round by going towards the red circles. So, the evader can only move between $s_2$ and $s_3$. In this case, the pursuers move first towards the yellow circle, following the line connecting them. While the pursuers move, the evader cannot move towards $s_1$ or $s_4$. Once the pursuers reach the yellow circles, the evader can no longer move toward $s_0$. If the evader is not stationary, the pursuers follow the line and move towards the light blue circle. When the pursuers reach light blue circles, all lanes will be blocked. If the evader remains stationary at $s_2$ or $s_3$, the pursuers must also move towards the light blue circles. Without loss of generality, let the evader be at $s_2$. The left pursuer blocks the moving lane to $s_4$, the right pursuer blocks the lane to $s_0$, and the lane to $s_3$ will lead to a losing situation. The evader can only move towards $s_1$. If the evader goes to $s_1$, the pursuer can return to yellow circles before the evader reaches $s_1$. If the evader reaches $s_1$ and the pursuers are at yellow circles, the evader can only go to $s_0$ or $s_4$. Either will result in a losing condition analyzed above.

This winning strategy requires the pursuers' velocity to be larger than $v_e\sin(\frac{\pi}{10})$ so that each pursuer can cover two moving lanes. Let $v_e$ be $1$. Then the pursuers need to have a velocity of approximately $0.31$. Usually, athletes can kick the ball to have a velocity of around $75$ to $90 km/h$. That means the pursuers need to have a velocity of around $25km/h$. Many professional footballers can achieve that speed, meaning two pursuers with a good strategy can win the keep-away game against five passers (in a kinematic sense). If the passers are allowed to move, the corresponding velocity bound of the pursuers needs to increase. 

\vspace{-0.05in}
\section{Winning Strategies and How to Find Them}
\vspace{-0.05in}

The winning strategies are difficult to find algebraically. The main challenge is the continuous decision space for the pursuers. The pursuers can win if they reach specific locations, but minor variations may lead to alternative results. Unless we can solve optimization or represent the decision boundary analytically, the algebraic representations may lead to incorrect results.

Given $n$ states and $m$ pursuers, we know there are $O(n^2)$ lanes. Since many of the lanes are adjacent geometrically, we could just check whether the $m$ pursuers can arrive at any lane in time to block the evader if the lane is used. If the pursuers can successfully {\em block} all lanes between two disjoint subsets of the $n$ states within the time for the evader to move between any two states, the evader cannot escape one of these subsets. If the pursuers can reduce the size of the blocked subset until it contains only a single state, the pursuers will be able to win. 

We refer to the pursuer locations that can block a subset of states as a {\em formation}. Let the velocity $v_e$ and $v_p$ be given and fixed. We first need to identify the possible coverage regions against subsets of states. As the positions of the states are given, the pursuer formations have fixed geometries. For simplicity, we label the states $s_1$ through $s_n$, where $s_1$ is on the positive $y$ axis, and the indices increments counter-clockwise.

\begin{algorithm}[b!]
    \caption{Formations (i)}\label{alg:formation_i}
    \SetKwInOut{Input}{input}
    \SetKwInOut{Output}{output}
    \Input{$n$, $m$, $i$}
    \Output{Pursuers geometry}
    $V_i\leftarrow p_1, \ldots, p_{i}$\;
    $\mathcal{L}\leftarrow$ Identify all lanes from set $V_i$\;
    sort $\mathcal{L}$ based on lengths\; % sort based on what? length:
    $F\leftarrow\emptyset$\;
    Compute the coverage region triangle for each $l$\;
    \For{$l\in L$} {
        Find overlapping triangles with $l$ with at least one common lane endpoint\;
        $p\leftarrow$ overlapping region (polygon)\;
        Remove corresponding $l$ from $L$\;
        $F\leftarrow F\cup o$\;
    }
    \Return $F$\;
\end{algorithm}

Given a set of $n$ states and $m$ pursuers, we can use Algorithm~\ref{alg:formation_i} to find the region of possible positions for the pursuers to block a set of $i$ adjacent states. As the evader travels along a straight line, and the pursuers can move in an arbitrary direction without acceleration bound, the region that can block a lane is in a shape of a triangle, or more precisely, a cone whose apex is the state from which the evader starts. The spanning angle of the cone is $2\mathrm{atan}{\frac{v_p}{v_e}}$. In 2D, we can intersect these triangles to find overlapping regions that can block multiple passing lanes successfully. In 3D, the cone intersection is also easily computable. 

% exist more precise way:
% for each passing lane, coverage region is a triangle;
% the overlapping region of all triangles are

% \begin{algorithm}
%     \caption{$\min\max d_{a, b}$)}\label{alg:minmaxd}
%     \SetKwInOut{Input}{input}
%     \SetKwInOut{Output}{output}
%     \Input{$d_{a, b}\forall o_a\in F_i, o_b\in F_j$}
%     \Output{$\min\max d_{a, b}$)}
%     \For{each $o_k\in F_j$} {
%         sort $d_{a, k}$\;
%     }
%     \While{true} {
%         Pick the largest $d_{a, k}$ for each $a$\;
%         \If{All $k$ are the same} {
%             \Return the minimum $d$ among all $d_{a, k}$\;
%         }
%         Find the $k$ with the smallest distance\;
%         Remove $d_{a, k}$ that has a different $k$\;
%     }
% \end{algorithm}

Once the formations are found against each $i$ that is no larger than $\lceil n/2\rceil$, we can check if two formations are {\em close enough} so that the pursuers can move from formation again $i$ states to a formation against $(i-1)$ states. If the answer is yes until $i = 2$, i.e., can successfully reduce the size of the blocking subset from $\lceil n/2\rceil$ to $1$, the pursuers can win. Pursuers may lose if the answer is no to any adjacent formation transitions. We can find the shortest distance $ d = \min\max d_{a, b}$ between arbitrary two polygons or polyhedrons ($p_a\in \mathcal{F}_i$ and $p_b\in\mathcal{F}_j$) given two formations $\mathcal{F}_i$ and $\mathcal{F}_j$. 
% Then, we can use Algorithm~\ref{alg:minmaxd} to find the minimized maximum distance between two formations. 

Directly finding $\min\max d_{a, b}$ is equivalent to a one-to-one assignment problem and is NP-hard. Greedy algorithms can be used to find potential solutions to the problem. In the proposed problem, as the pursuers are all bounded within a convex region, we only need to find whether the assignment can allow a successful transition with a given velocity. We can use Algorithm~\ref{alg:formation_dis} to test if the condition is satisfied. Instead of the complete assignment, we only need to find a valid assignment with a maximum distance below the allowed value. Therefore, whether the assignment is optimal is not critical. In Algorithm~\ref{alg:formation_dis}, we use a greedy approach to find if the constraint can be satisfied. \vspace{0.3in}

\begin{algorithm}[h]
    \caption{FormationsTrasition ($F_i$, $F_j$)}\label{alg:formation_dis}
    \SetKwInOut{Input}{input}
    \SetKwInOut{Output}{output}
    \Input{$F_i$, $F_j$}
    \Output{$d\in \mathbb{R}$}
    \For{$p_a\in F_i, p_b\in F_j$} {
        $d_{a, b}\leftarrow \|p_a - p_b\|_2$\;
    }
    $t\leftarrow$ Minimum traveling time from set of $i$ states to set of $j$ states\;
    $l\leftarrow t\cdot v_c$\;
    % Now, need to test 
    \For{each $p_a\in F_i$} {
        $D_a\leftarrow$ sort $d_{a, k}$ $\forall p_k\in F_j$\;
        Remove $d_{a, k}$ if $d_{a, k} > l$\;
    }
    \While{exist $p_a$ not assigned and $D_a\neq\emptyset$} {
      Find the $p_a$ with smallest $|D_a|$\;
      Pick smallest $d_{a, k}\in D_a$\;
      Remove $D_a$ and all $d_{\cdot, k}$ in other $D_{\cdot}$\;
    }
    \uIf {All $p_a$ are assigned} {
      \Return true\;
    }\Else {
      \Return false\;
    }
\end{algorithm}

\newpage

The overall process can be written as Algorithm~\ref{alg:interceptor_win}. 
\begin{algorithm}[t]
    \caption{Pursuer winning strategy}\label{alg:interceptor_win}
    \SetKwInOut{Input}{input}
    \SetKwInOut{Output}{output}
    \Input{$m$, $n$}
    \Output{Whether the pursuers can win}
    $i\leftarrow \lceil\frac{n}{2}\rceil$\;
    \While{$i > 0$} {
        $\mathcal{F}_i\leftarrow Formations(i)$\;
        $\mathcal{F}_{i-1}\leftarrow Formations(i-1)$\;
        $C \leftarrow\mathrm{FormationsDistances}(\mathcal{F}_i, \mathcal{F}_j)$\;
        \uIf{$C$ is true} {
            continue\;
        }\Else{
            \Return false\;
        }
    }
    \Return true\;
\end{algorithm}

The above analysis holds even when the states are not placed uniformly on a circle. The only difference is that the pursuer formations and state placements are no longer symmetric. The fundamental condition for winning the game is the same. The main challenge comes when the pursuers have acceleration bound, and the coverage regions are no longer circles.

Figures~\ref{fig:five-vs-2} to~\ref{fig:seven-nonsym-vs-3} show the evader moving lanes (red dotted), the pursuer formations against different number of states (light blue region), and the transitions among formations. Even when the states are not placed on a convex region, the same algorithm applies, as shown in Figure~\ref{fig:seven-nonsym-vs-3}. In the shown examples, the velocities for the pursuers are larger than $v_e\sin(\frac{\pi}{2n})$, so the coverage region against a single passer is not a single line~but~a~polygon. 

In Figure~\ref{fig:five_against_two}, we can see that the transition formations overlap. Therefore, it is always possible for the pursuers to change formations and slowly reduce the subset of states the pursuers can block. In Figures~\ref{fig:six-vs-3}, ~\ref{fig:seven-vs-3}, and~\ref{fig:seven_transitions_merge}, the formation regions have gaps, meaning in order for pursuers to win, they need to be able to finish the formation transition within a pass. In these two examples, the pursuers have large enough velocity to make the formation transition feasible within traveling time for the evader. Thus, the pursuers can win in these scenarios. However, if we reduce the pursuers' velocity to be close to $v_e\sin(\frac{\pi}{2n})$, the gap between the formations increases. We have the next Lemma.

\begin{figure*}[h]
    \centering
    \subfigure[Formation against $3$ states. ]{
    \includegraphics[width=1.55in]{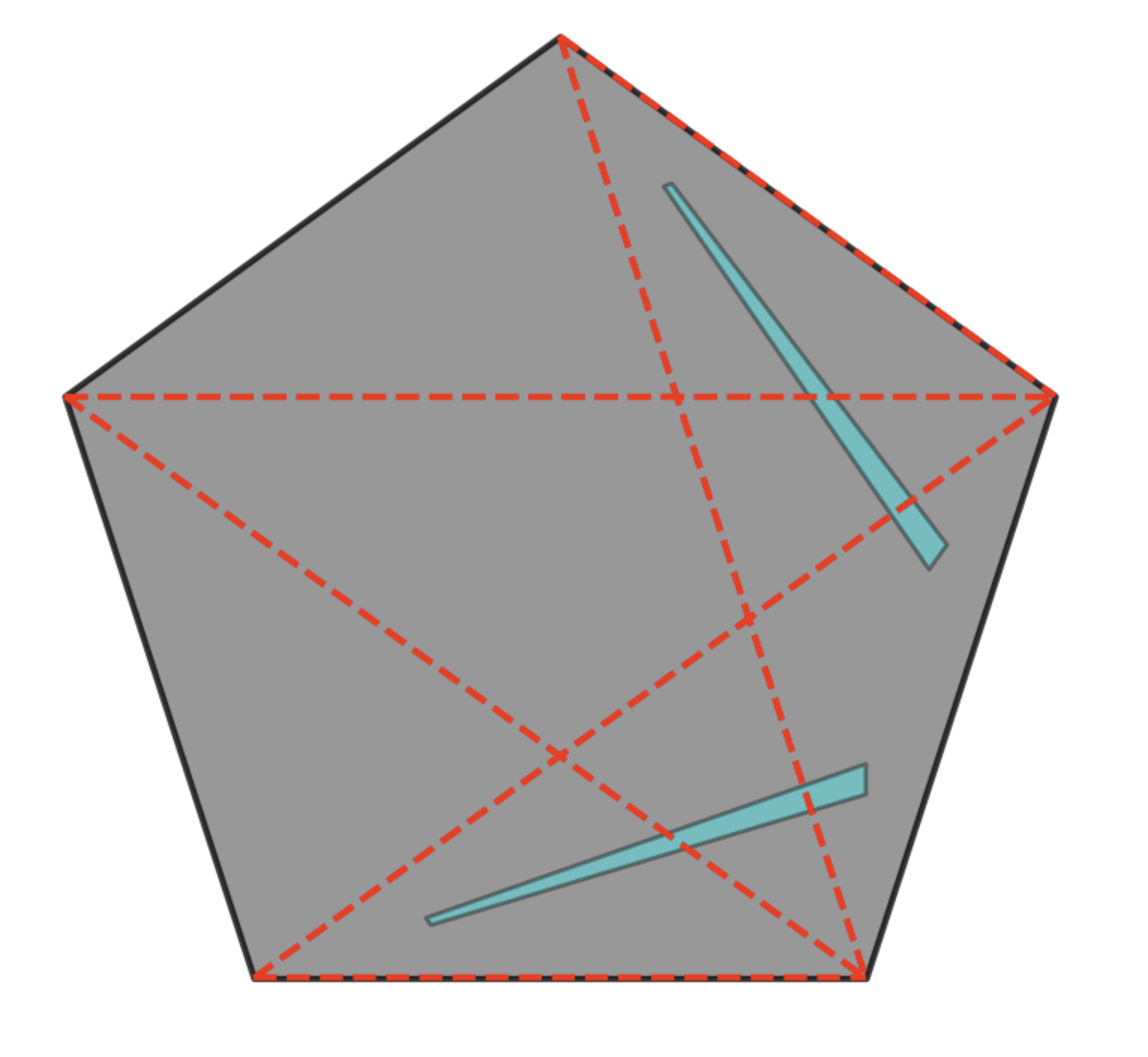}
    \label{fig:five-2-vs-3}
    }
    \hspace{-0.1in}
    \subfigure[Formation against $2$ states. ]{
    \includegraphics[width=1.55in]{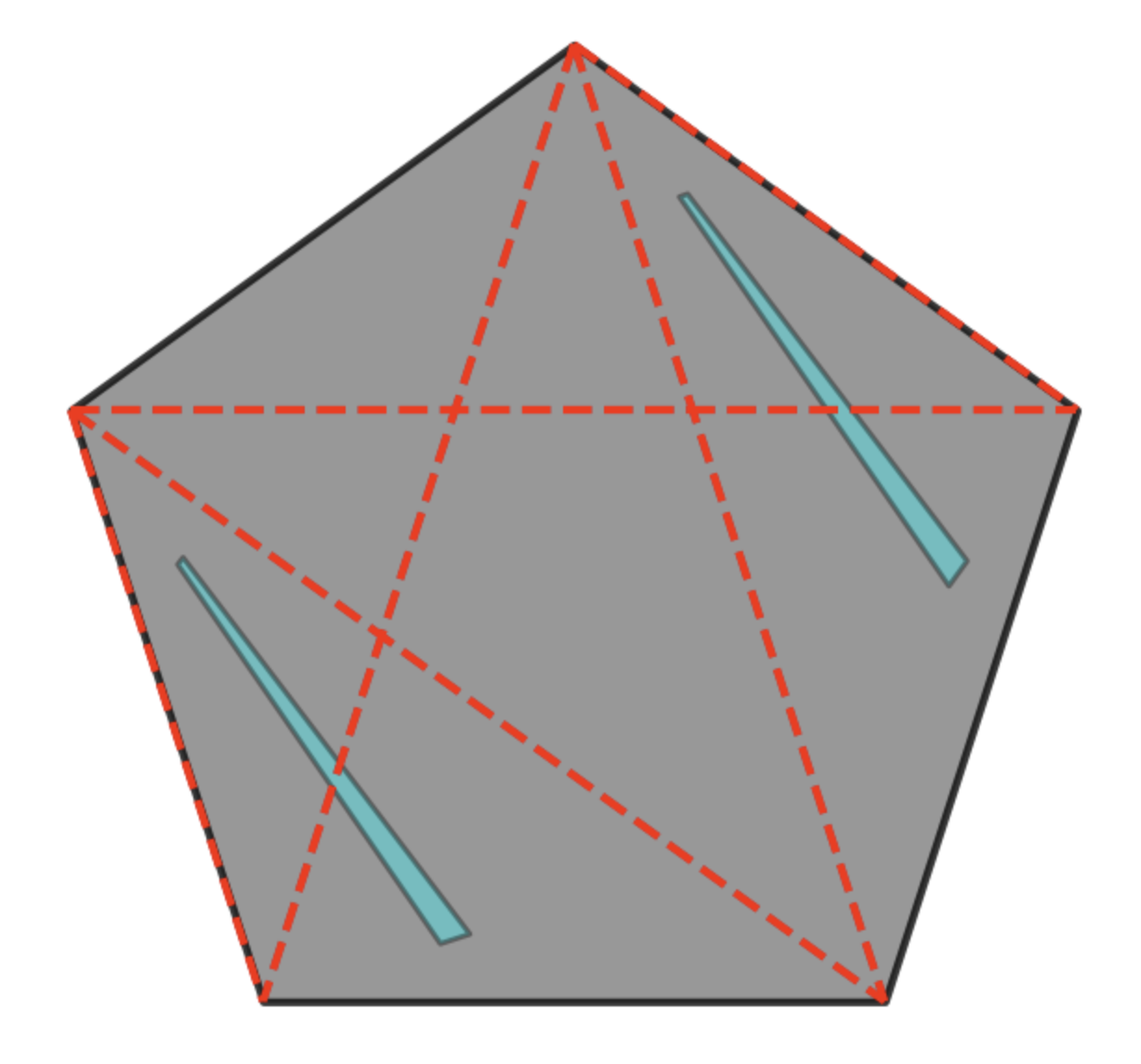}}
    \hspace{-0.1in}
    \subfigure[Formation against $1$ state. ]{
    \includegraphics[width=1.55in]{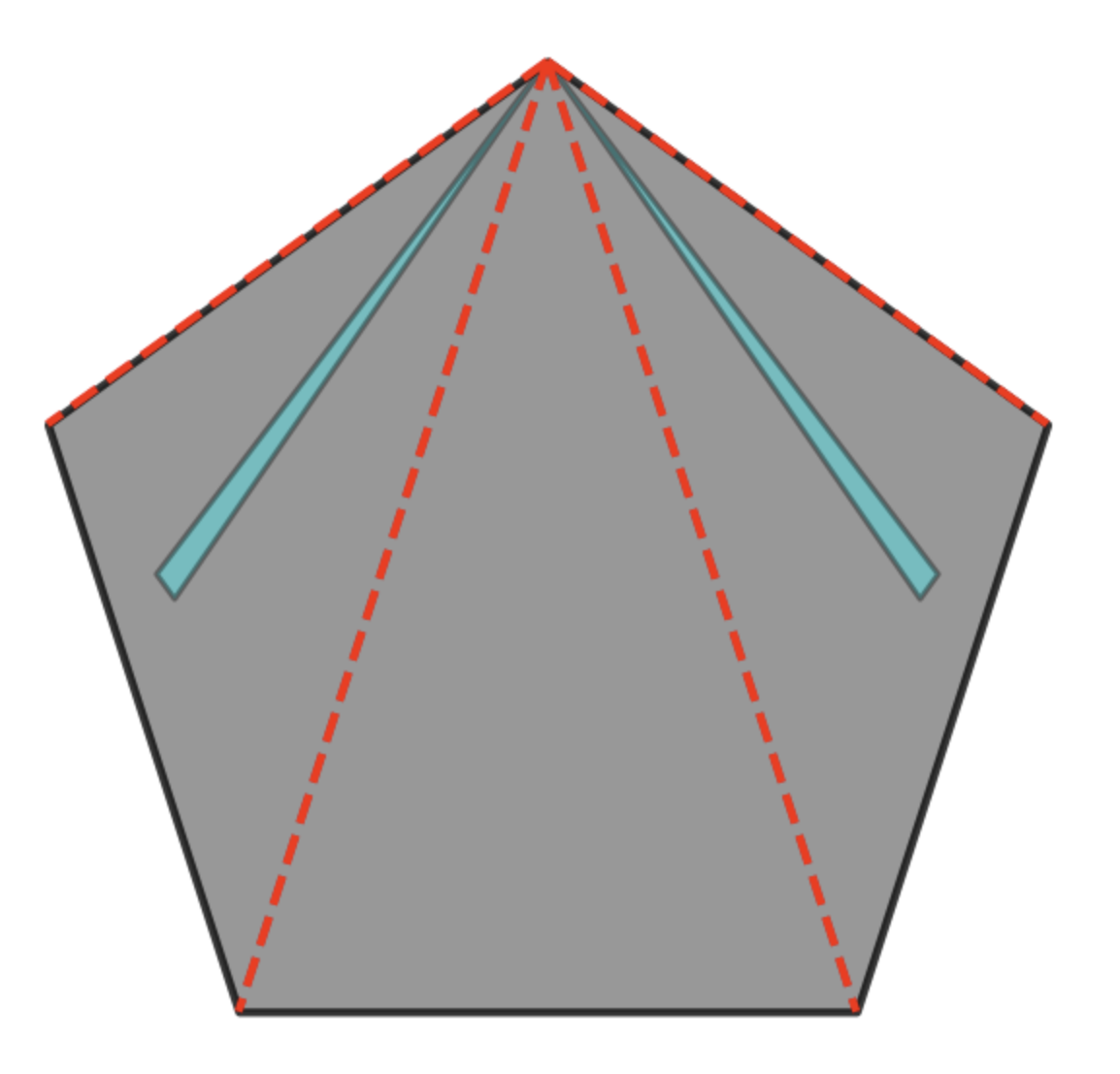}}
    \hspace{-0.1in}
    \subfigure[Formation transitions.]{
    \includegraphics[width=1.55in]{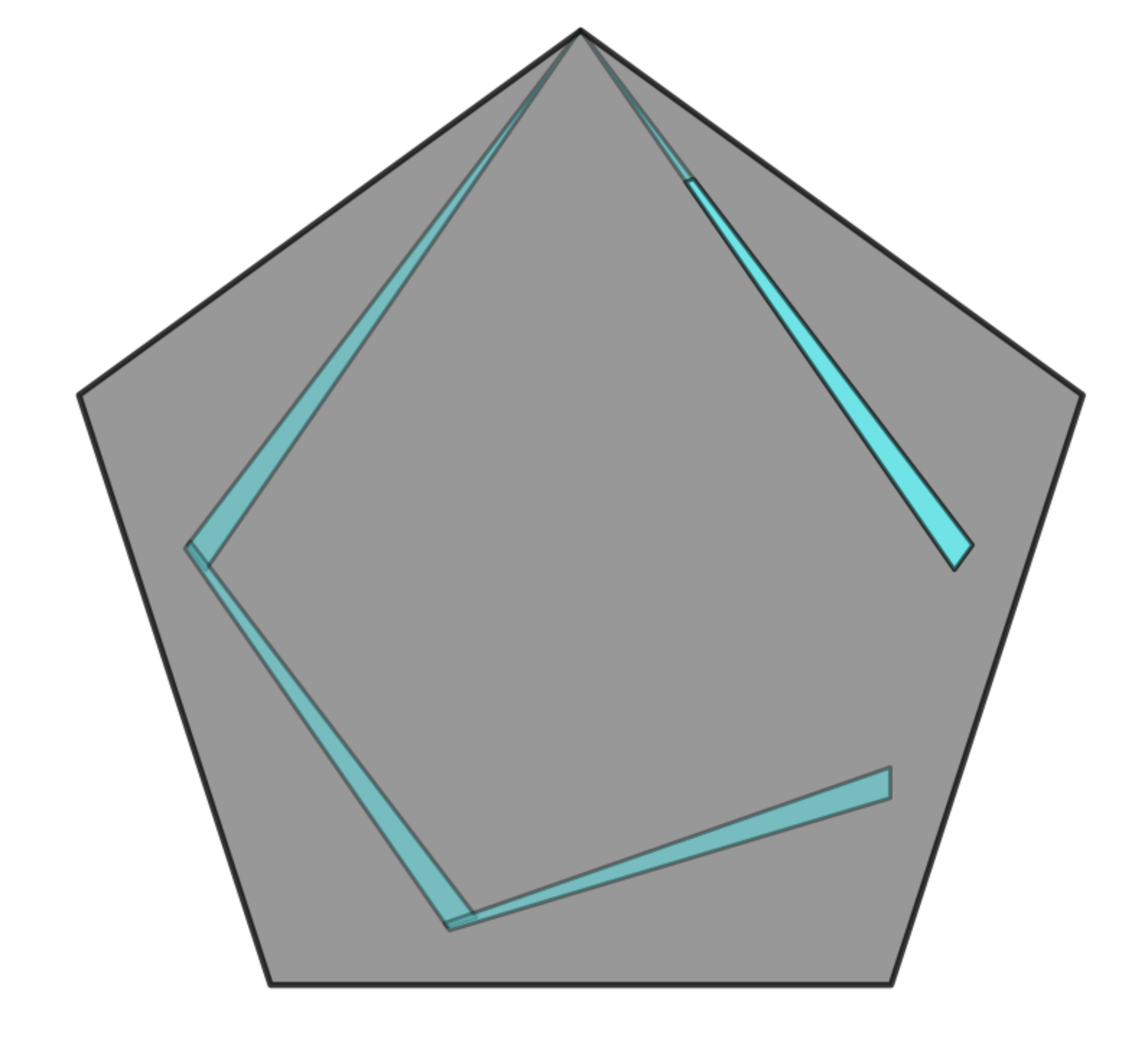}}
    \caption{With five states and two pursuers, different formations and transitions among formations.}
    \label{fig:five-vs-2}
\end{figure*}

\begin{figure*}
    \centering
    \subfigure[Formation against $3$ states. ]{
    \includegraphics[width=1.55in]{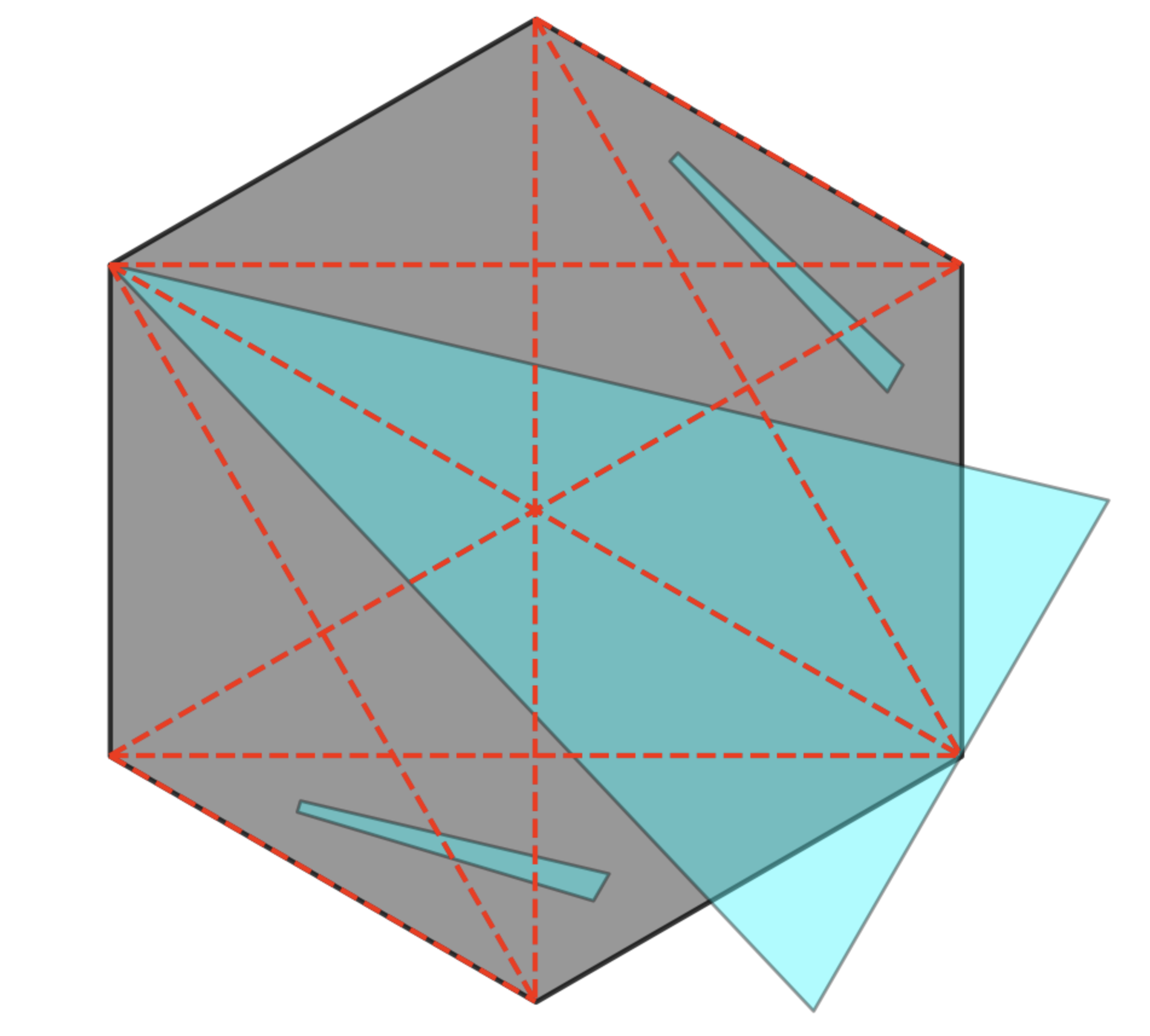}
    % \label{fig:five-2-vs-3}
    }
    \hspace{-0.1in}
    \subfigure[Formation against $2$ states. ]{
    \includegraphics[width=1.45in]{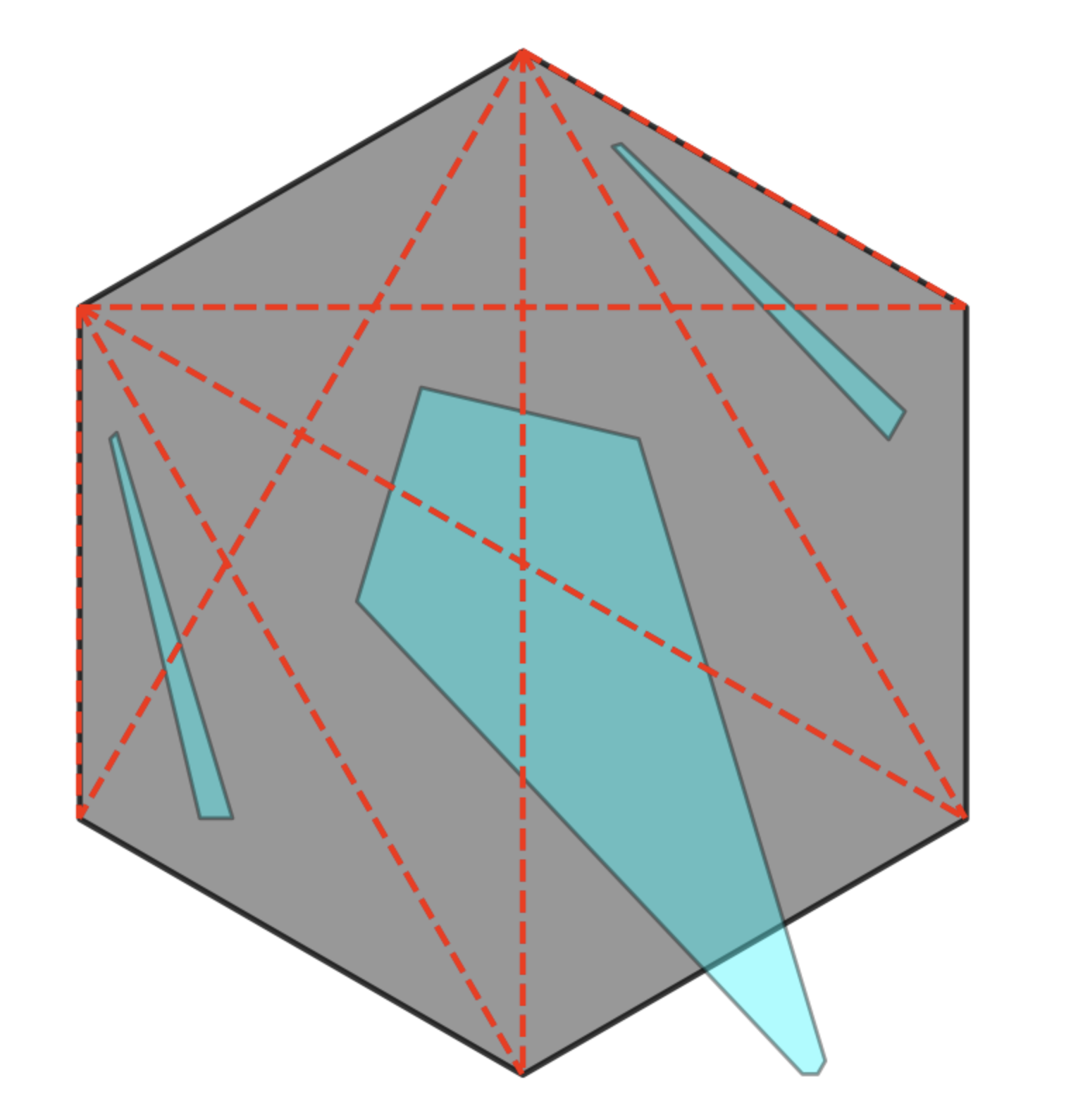}}
    \hspace{-0.1in}
    \subfigure[Formation against $1$ state. ]{
    \includegraphics[width=1.55in]{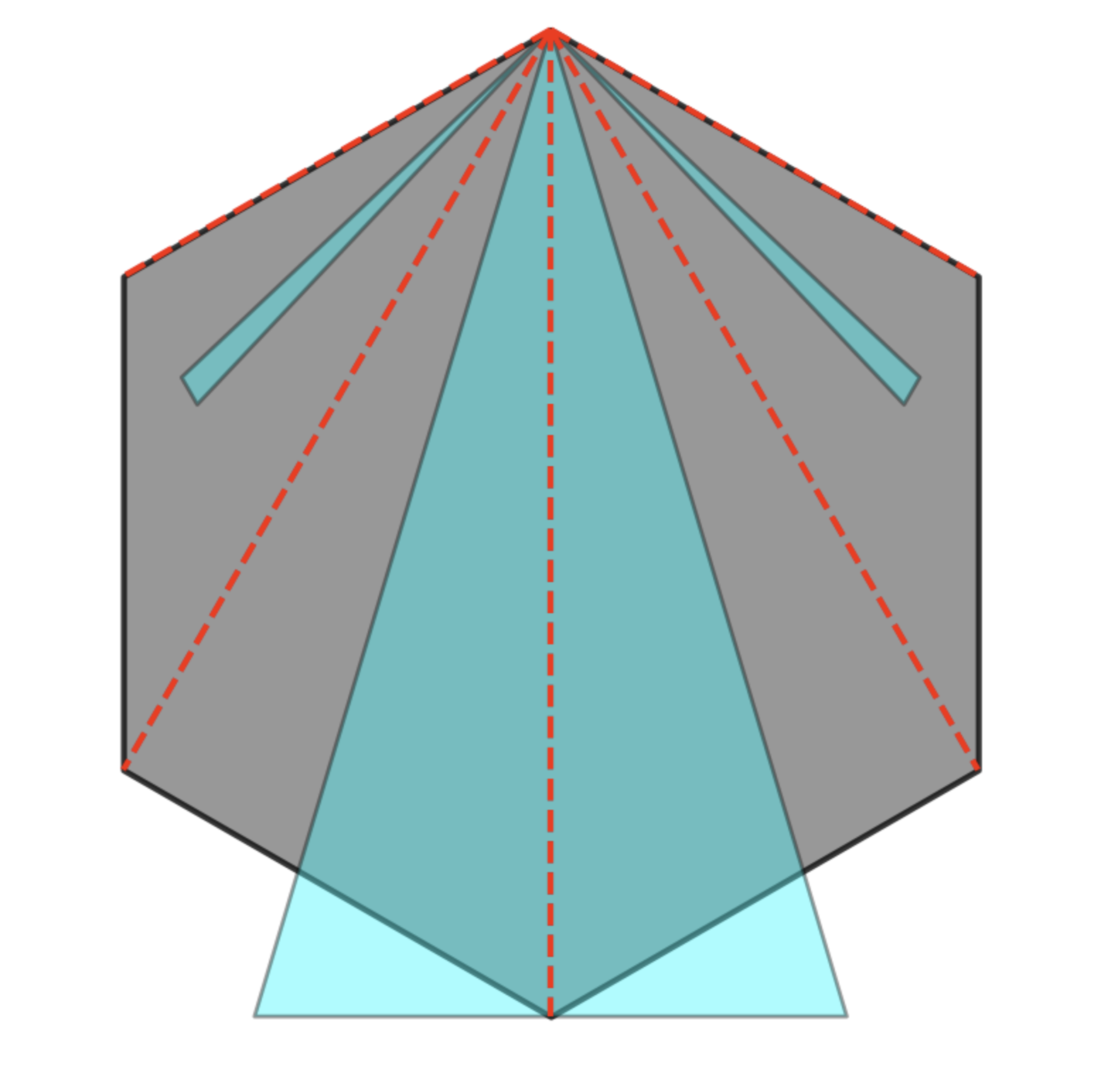}}
    \hspace{-0.1in}
    \subfigure[Formation transitions.]{
    \includegraphics[width=1.55in]{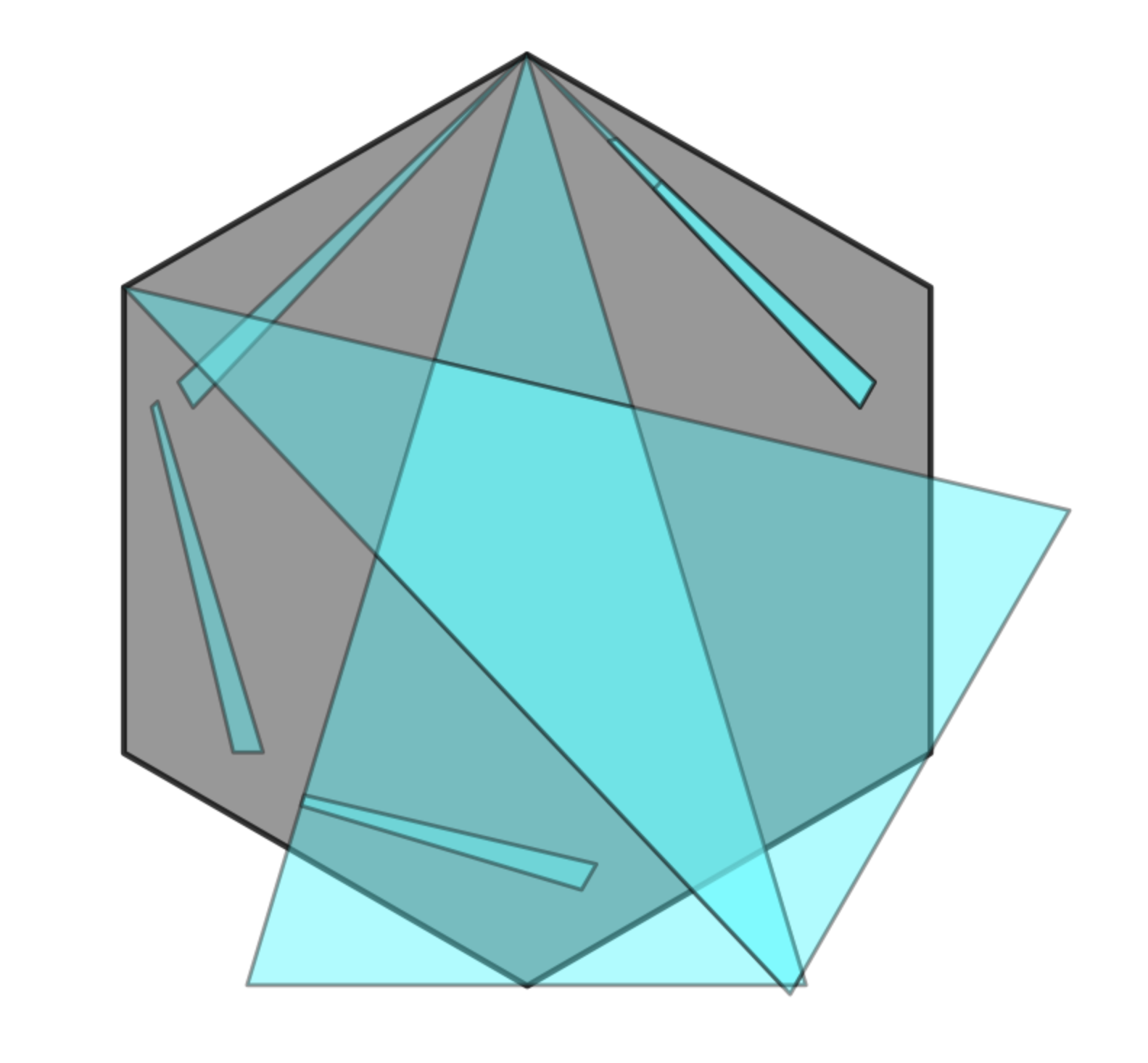}}
    \caption{With six states and three pursuers, different formations and transitions among formations.}
    \label{fig:six-vs-3}
\end{figure*}

\begin{figure*}
    \centering
    \subfigure[Formation against $4$ states. ]{
    \includegraphics[width=1.55in]{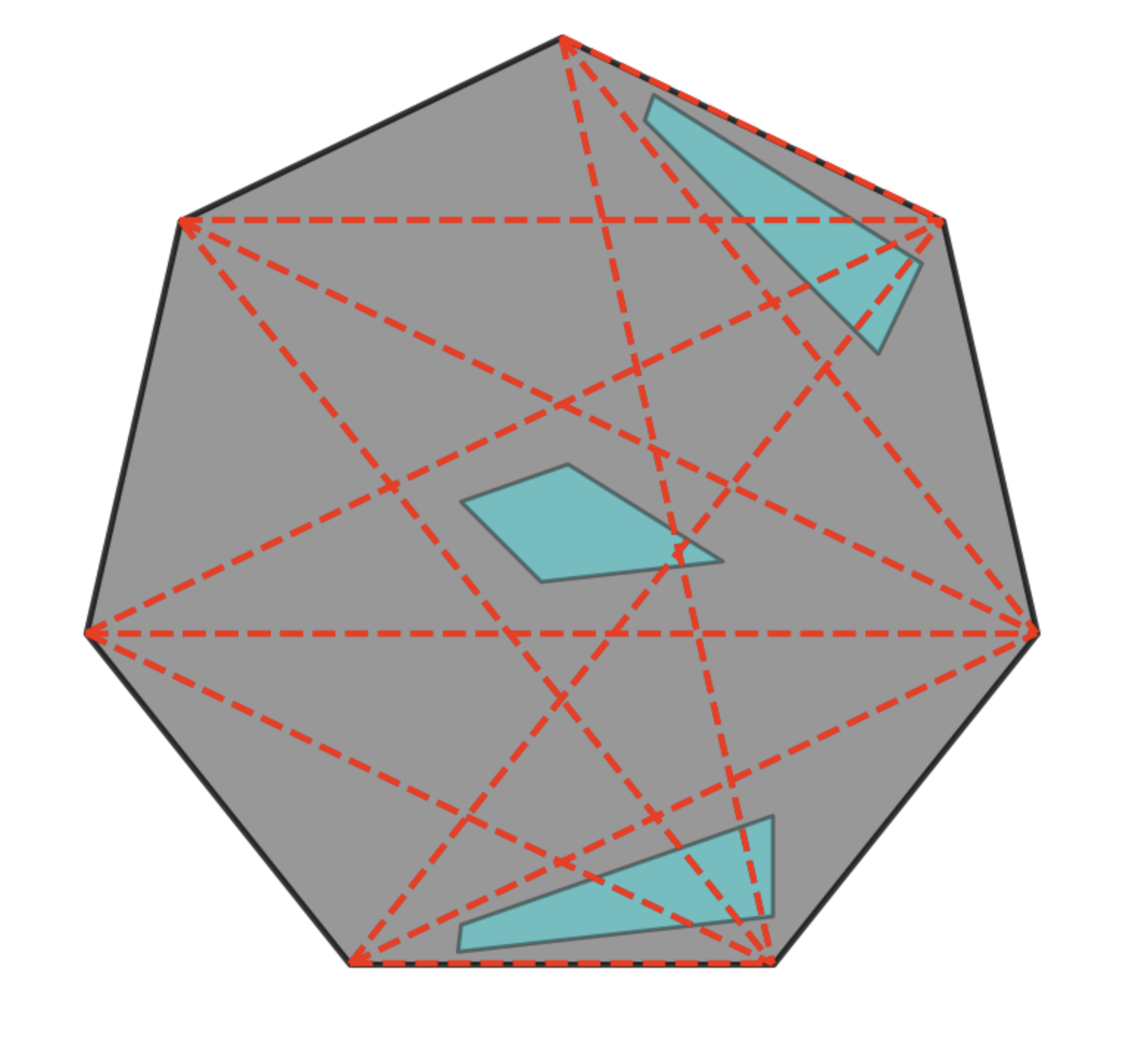}
    % \label{fig:five-2-vs-3}
    }
    \hspace{-0.1in}
    \subfigure[Formation against $3$ states. ]{
    \includegraphics[width=1.55in]{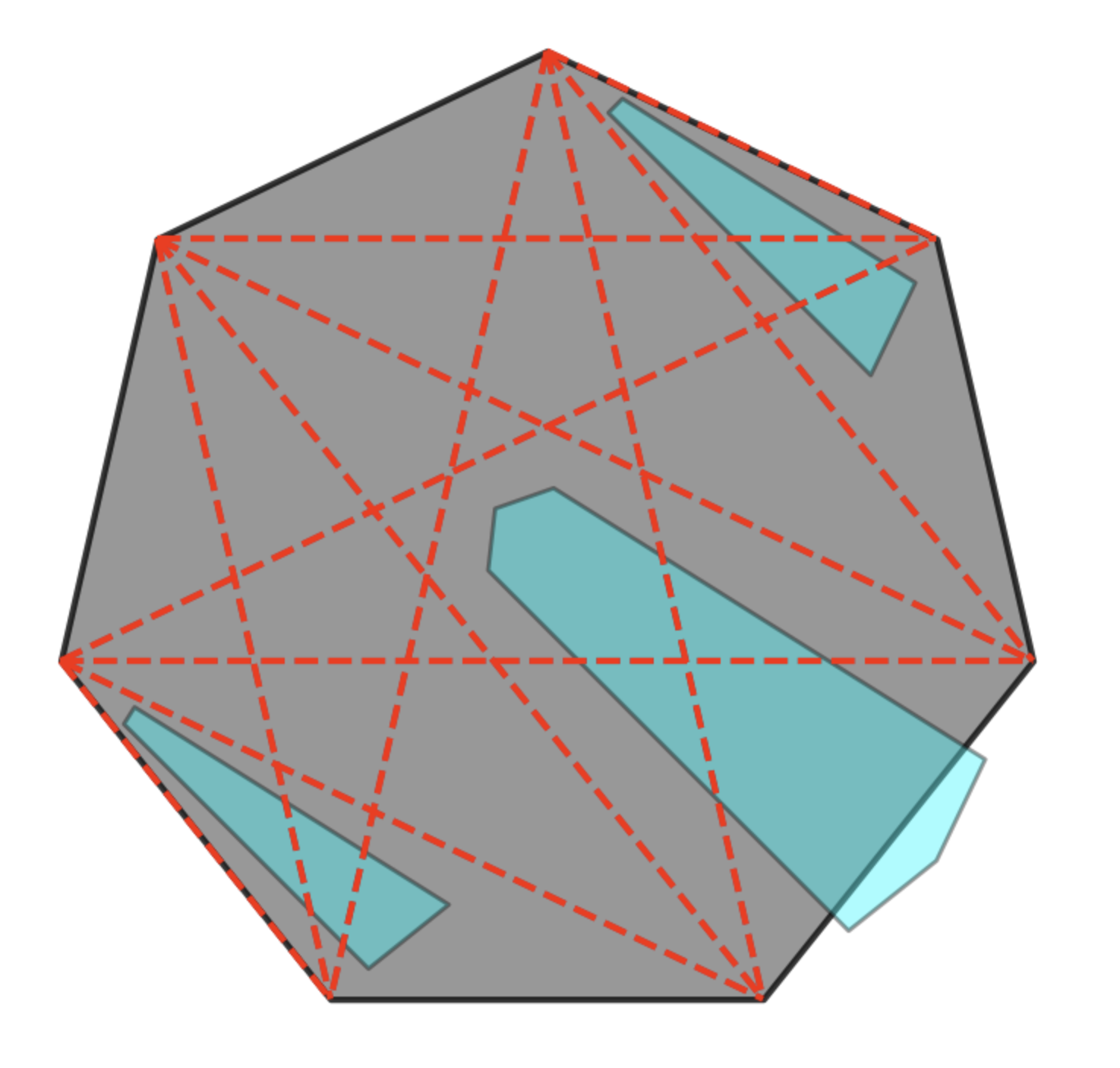}}
    \hspace{-0.1in}
    \subfigure[Formation against $2$ states. ]{
    \includegraphics[width=1.55in]{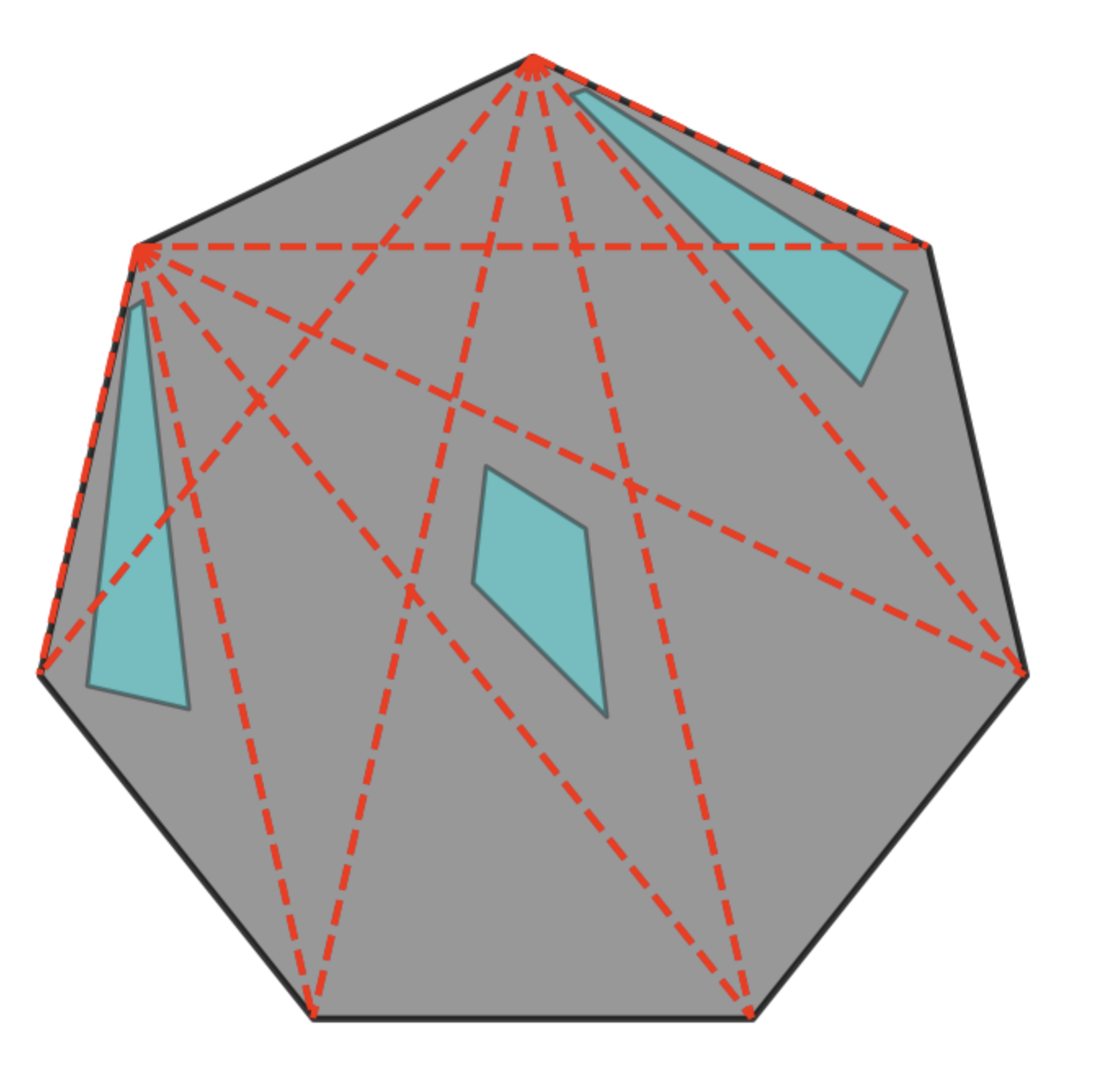}}
    \hspace{-0.1in}
    \subfigure[Formation against $1$ state. ]{
    \includegraphics[width=1.55in]{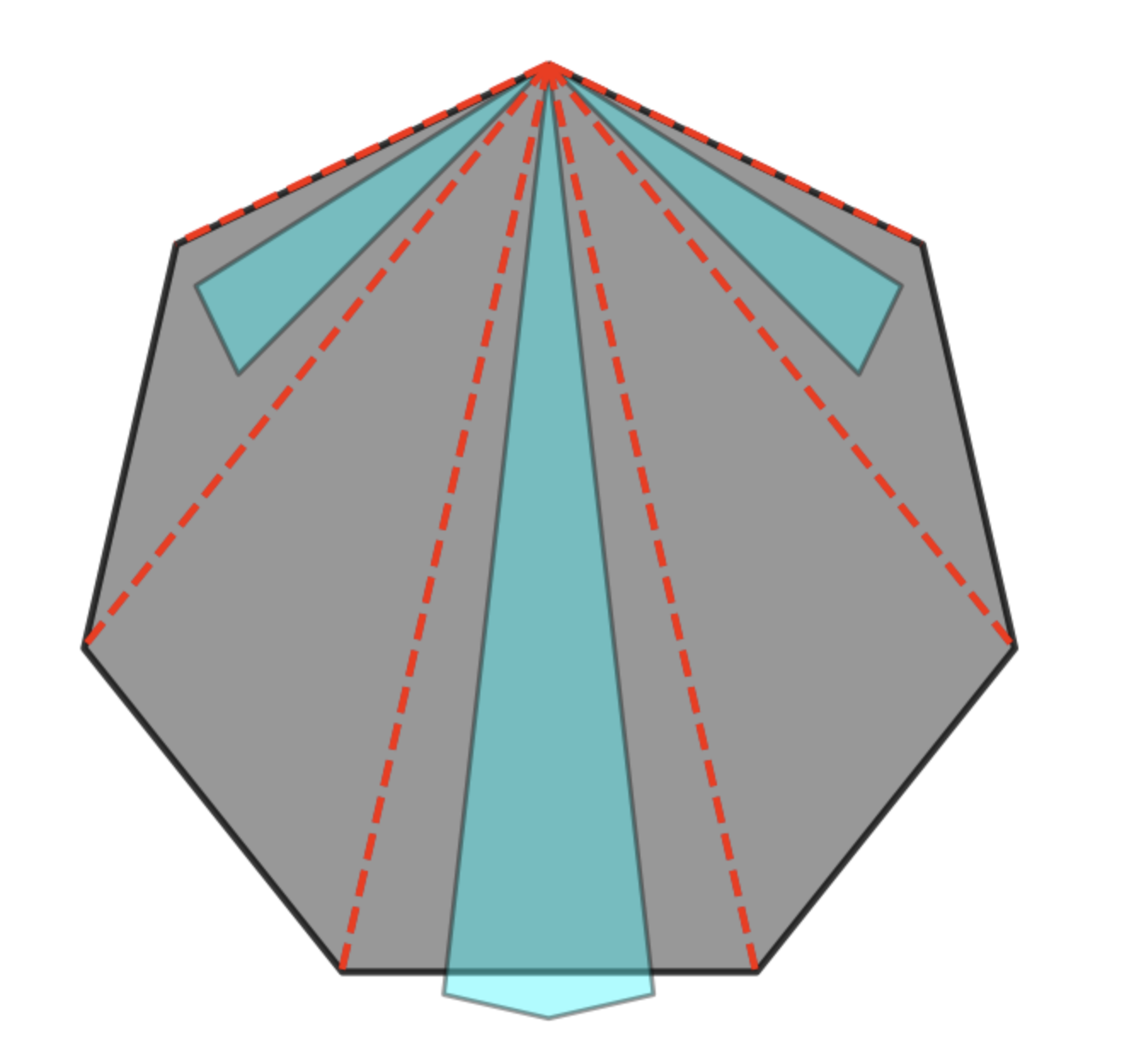}}
    \caption{With seven states and three pursuers, different formations and transitions among formations.}
    \label{fig:seven-vs-3}
\end{figure*}

\begin{figure*}
    \centering
    \subfigure[Formation against $4$ states. ]{
    \includegraphics[width=1.55in]{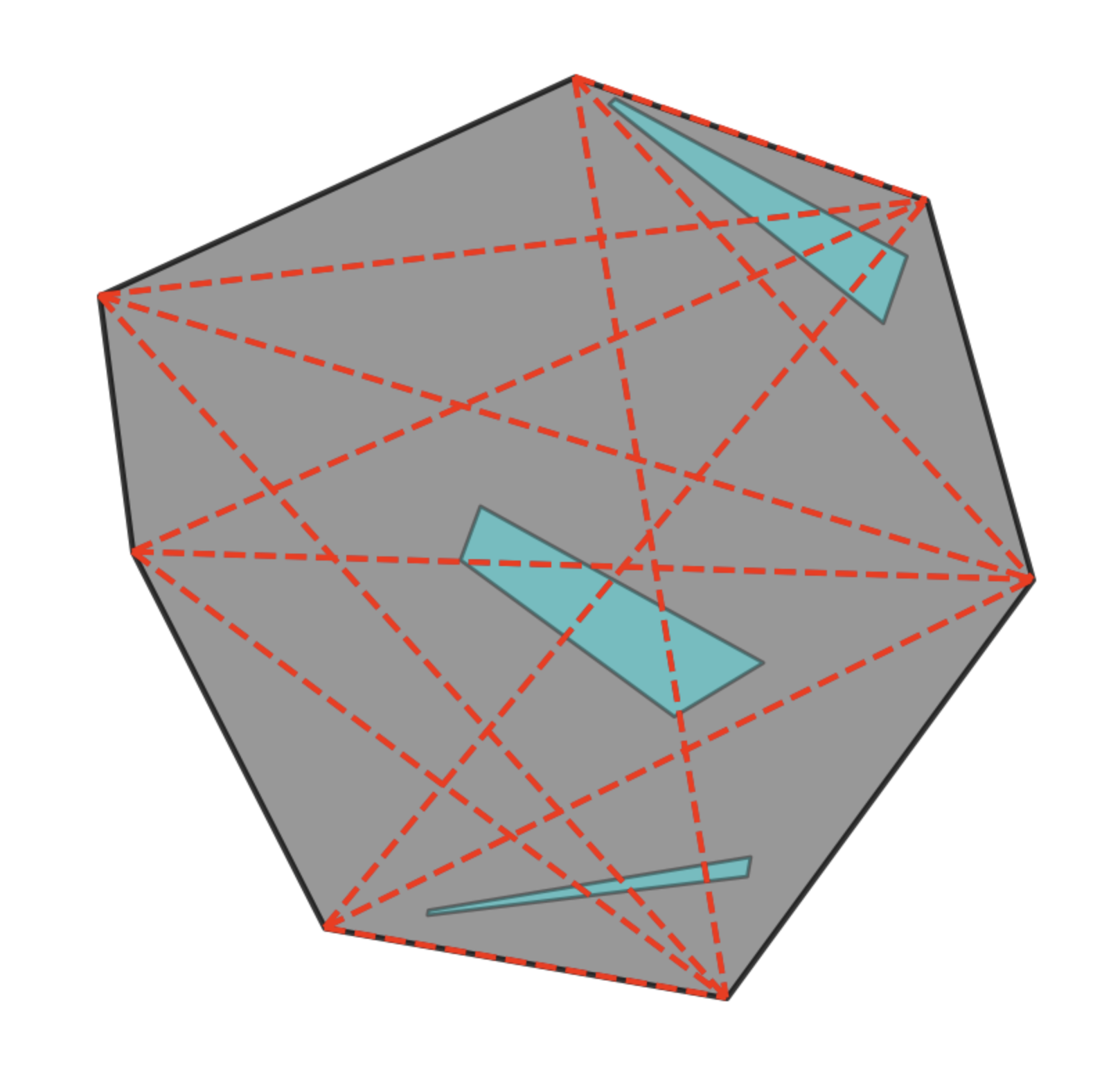}
    % \label{fig:five-2-vs-3}
    }
    \hspace{-0.1in}
    \subfigure[Formation against $3$ states. ]{
    \includegraphics[width=1.55in]{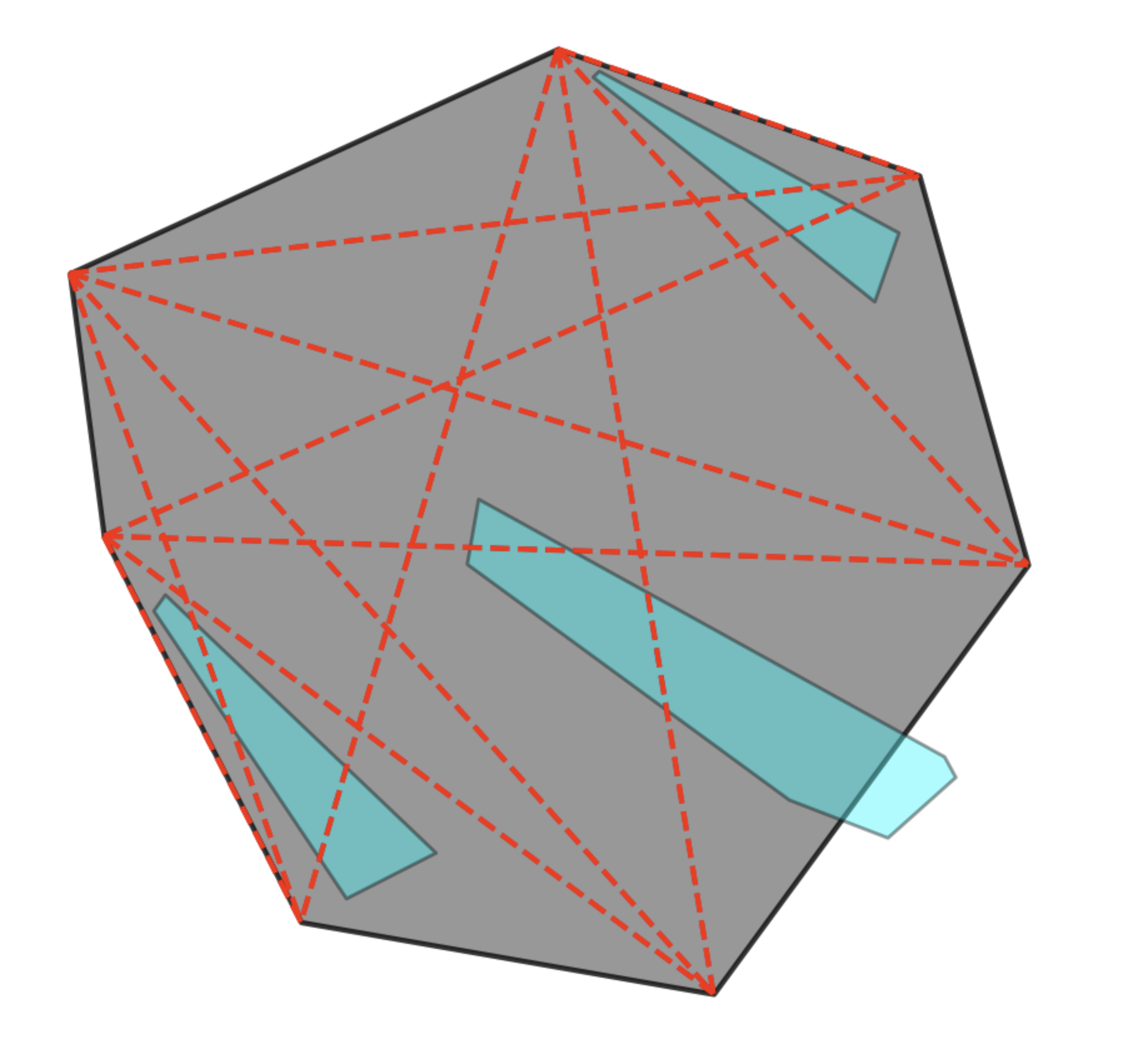}}
    \hspace{-0.1in}
    \subfigure[Formation against $2$ states. ]{
    \includegraphics[width=1.55in]{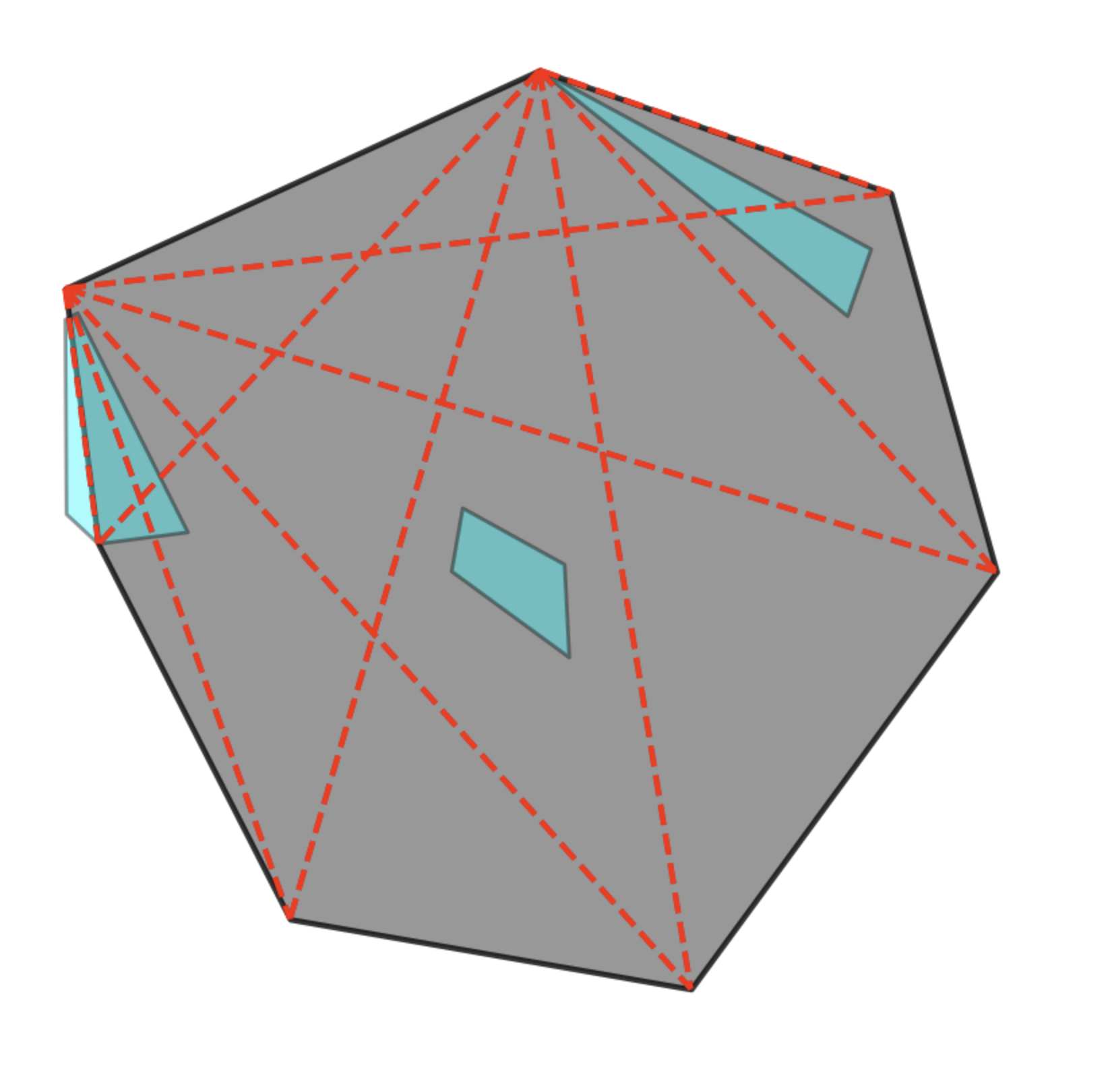}}
    \hspace{-0.1in}
    \subfigure[Formation against $1$ state. ]{
    \includegraphics[width=1.55in]{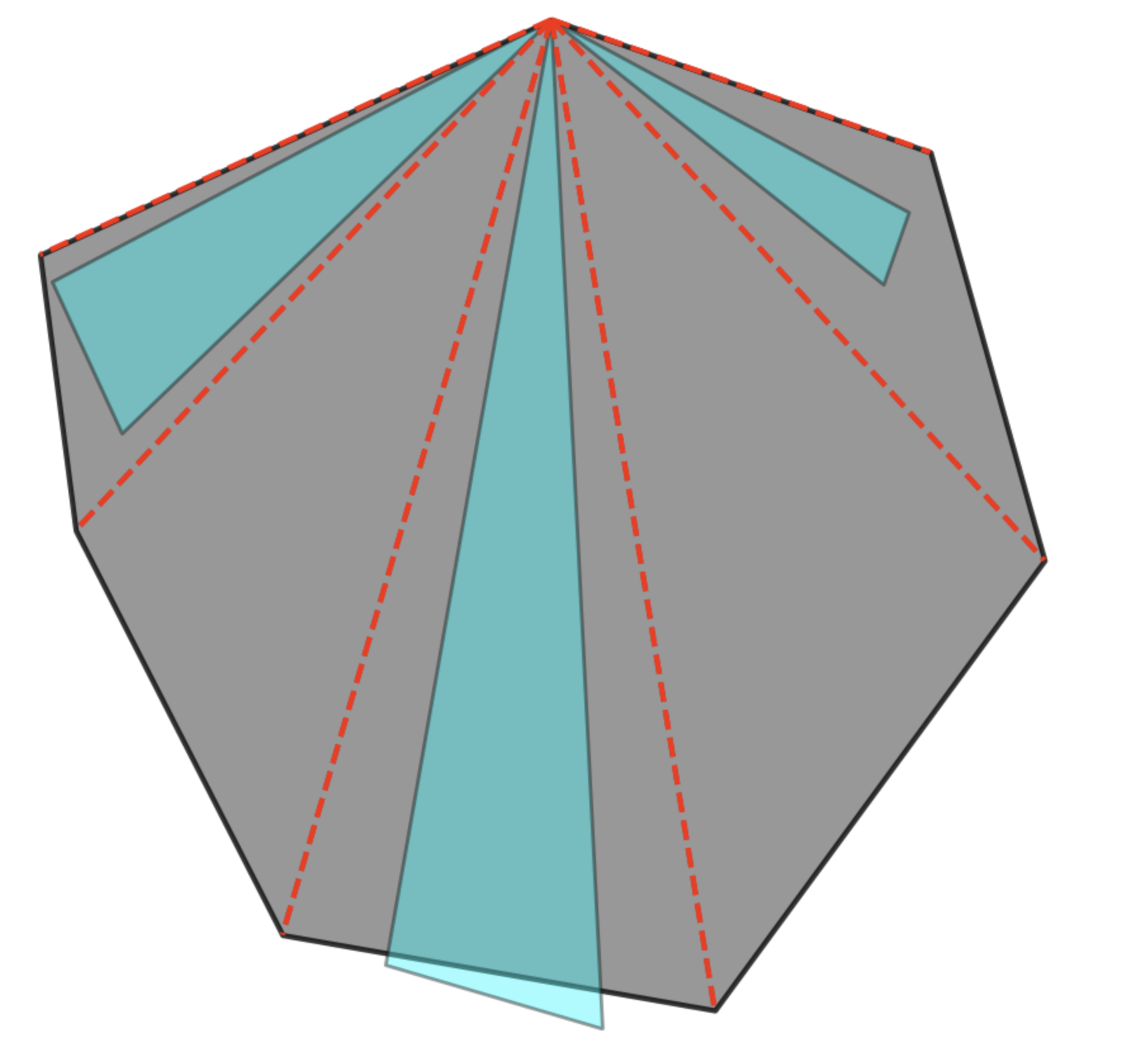}}
    \caption{With seven unevenly placed states and three pursuers, different formations and transitions among formations.}
    \label{fig:seven-nonsym-vs-3}
\end{figure*}

\begin{figure}
\centering
\subfigure[Formation transitions.]{
\includegraphics[width=1.55in]{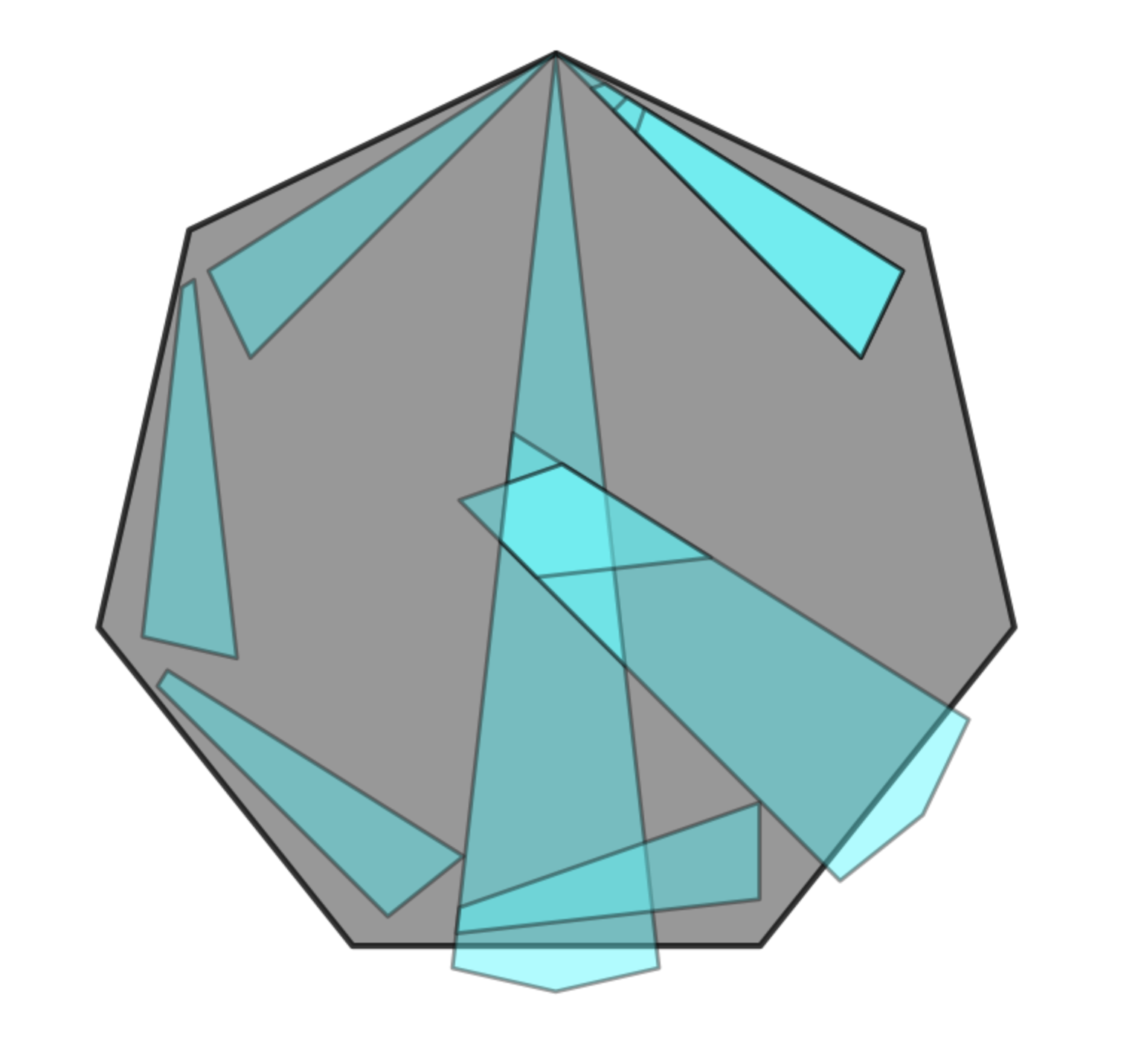}}
\subfigure[Formation transitions.]{
\includegraphics[width=1.55in]{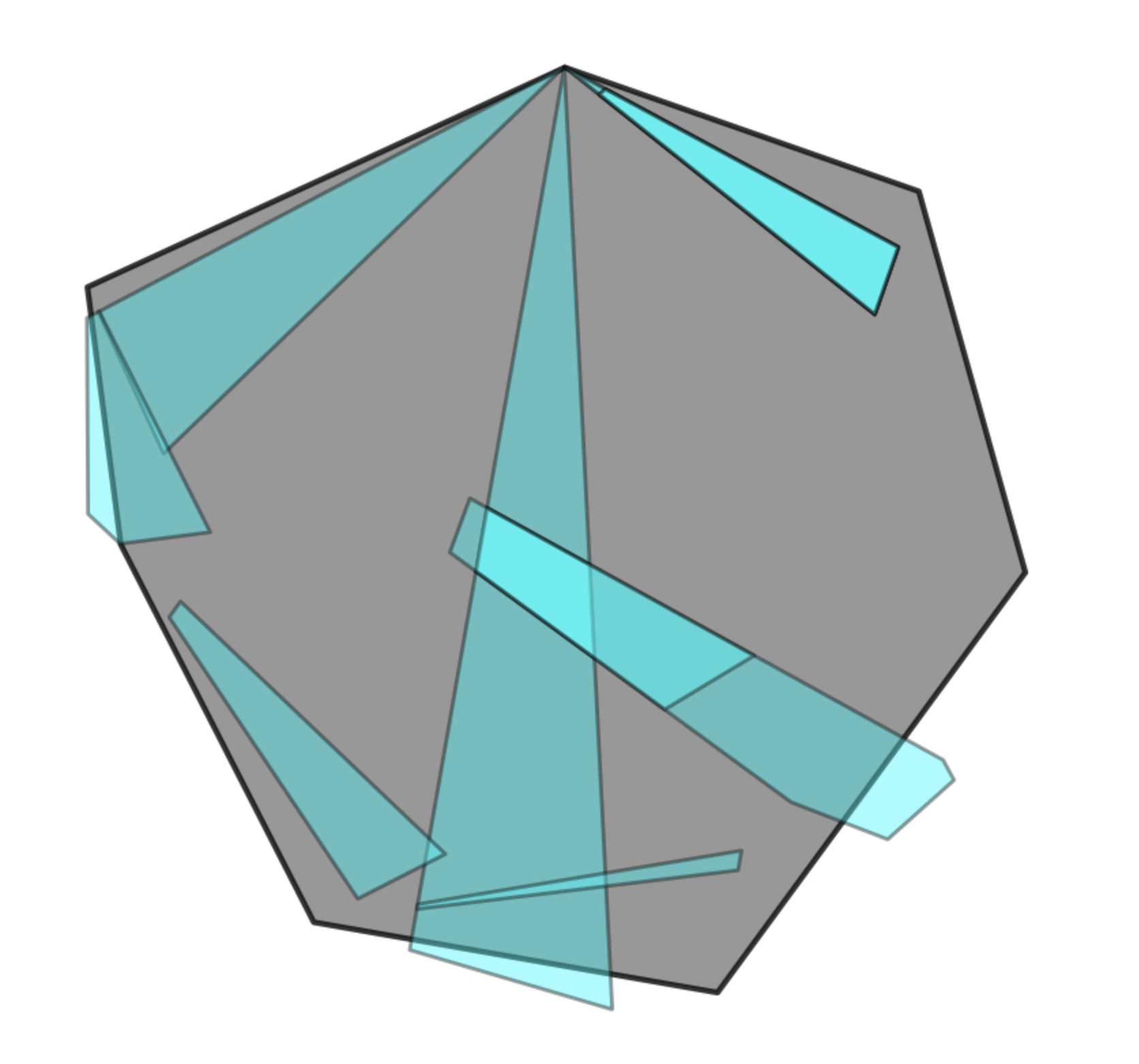}}
\caption{Formation transitions among formations for three pursuers against seven states. }
\label{fig:seven_transitions_merge}
\end{figure}

\begin{lemma}
Given $n$ states and $m$ pursuers in the proposed game where $n$ states are evenly placed~on~a circle, if each pursuer has velocity $v_e\sin(\frac{\pi}{2n})$ and $m$ is just sufficient to form a coverage against a single state, then $m$ pursuers cannot win the game~when~$n\geq 7$. 
\end{lemma}

\begin{proof}
The existence of a winning strategy is determined by whether a transition can happen, given the shortest traveling time among the blocked subset of states. So, let us assume the polygon has edge length $l$. Then, the transition distance $l_t$ needs to be no larger than $l\frac{v_c}{v_p} = l\sin(\frac{\pi}{2n})$. 

\begin{figure}[h]
\centering
\includegraphics[width=3.5in]{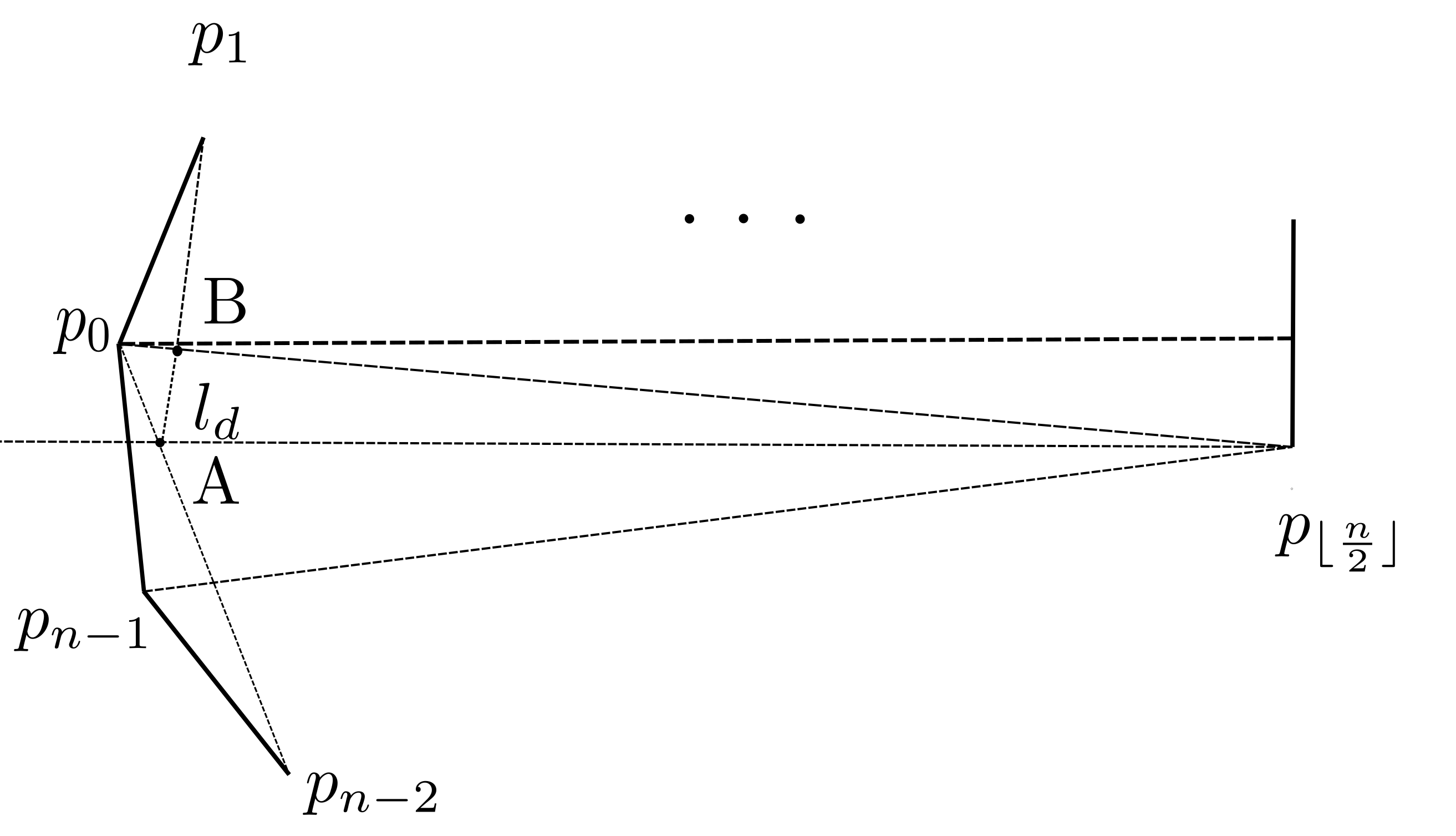}
\caption{The formation transition distances. }
\label{fig:triangles}
\end{figure}

At the same time, let us consider $p_0$, then $l_t$ is determined by the bisecting line from $\lfloor \frac{n}{2}\rfloor$ to $p_0$ and $p_{n-1}$, and the bisecting line from $p_0$ to $p_{n-1}$ and $p_{n-2}$. The intersection point of these two lines give the nearest point from $\lceil \frac{n}{2}\rceil$ formation to the $\lfloor\frac{n}{2}\rfloor$ formation. We need to have $l_t\leq l\frac{v_{c}}{v_{p}}$, which at limit is $l\sin(\frac{\pi}{2n})$. Then, we can apply sine theorem in triangles $p_0AB$ and $p_{\lfloor\frac{n}{2}\rfloor}AB$ shown in Figure~\ref{fig:triangles}, denote $\angle p_{\lfloor\frac{n}{2}\rfloor}p_0A = \alpha$, $
\angle p_0AB = \beta$, and $\angle BAp_{\lfloor\frac{n}{2}\rfloor} = \gamma$. We have $\alpha = \frac{\pi}{2} - \frac{2\pi}{n}$, and $\beta + \gamma = \frac{\pi}{2} + \frac{3\pi}{2n}$. 
Then the satisfying condition is: 
\begin{eqnarray}
% \frac{1}{2\sin(\frac{\pi}{2n})} - \sin(\frac{\pi}{2n}) \geq \sin(\frac{5\pi}{2n})
\sin\gamma \geq \frac{1}{2\sin(\frac{\pi}{2n})} - \sin(\frac{\pi}{2n})\\
\sin\gamma \leq \sin(\frac{\pi}{2} + \frac{3\pi}{2n})
\end{eqnarray}
When $n \geq 7$, the above conditions cannot be met. If the pursuer velocity is only $v_e\sin(\frac{\pi}{2n})$, the evader will win. 
\end{proof}

From the above lemma, we found that the necessary velocity to block against a single state decreases when the number of states increases. However, if the pursuers only maintain the necessary velocity to block a single state, the evader has strategies to win. Given the fixed evader velocity, the number of state increases, the number of needed pursuers needs to increase to form successful coverage against the states. At the same time, the minimum velocity for the pursuers to win does not decrease as much. It is possible to find the pursuer velocity limit, but the limit involves complex trigonometry computations. Borrowing Figure~\ref{fig:triangles}, let us denote $\alpha = \angle p_{\lfloor\frac{n}{2}\rfloor}p_0A$, and $\gamma = \angle BAp_{\lfloor\frac{n}{2}\rfloor}$, we have 
\begin{eqnarray}
\alpha = \frac{\pi}{2} - \frac{2\pi}{n},\ \gamma = \frac{\pi}{n} - \mathrm{asin}\frac{v_c}{v_p} \geq 0\\
c = \frac{l}{2\sin(\frac{\pi}{2n})}, \ \frac{c\cdot v_p}{l\cdot v_c} \leq \frac{1}{\tan\alpha} + \frac{1}{\tan\gamma}
\end{eqnarray}
For $n = 7$, we have $\frac{v_p}{v_e}\geq 0.323$ where minimum velocity for blocking against a single passer $\lim\frac{v_p}{v_e} = \sin{\frac{\pi}{2n}}\approx 0.223$, for $n = 9$, we have $\frac{v_p}{v_e}\geq 0.262$ where $\lim\frac{v_p}{v_e} \approx 0.174$, and for $n = 11$ we have $\frac{v_p}{v_e}\geq 0.223$ where $\lim\frac{v_p}{v_e} \approx 0.142$, etc. We can see that the minimum velocity drops much slower compared to the velocity limit for blocking against a single state. 

\section{Implications and Extensions}

In the proposed game, the placement (locations) of the pursuers are critical to whether they can win the game. Perhaps counter-intuitive, the best strategy usually is {\bf NOT} to move towards the current passer holding the ball but towards locations to force the evader to move to some state against which the pursuers can block all the potential moving lanes. The proposed method is correct even when the states do not form a convex shape. The proposed geometric procedure only requires the moving lanes to be straight lines. The proposed procedure can also be extended to 3D, where the intersecting geometry becomes cones. 

We can also visualize the game a bit differently: let the pursuers start at some locations and extrude the polygon into 3D. We can draw a cone for each pursuer whose apex is the current location of the corresponding pursuer. Each round, we grow a new cone within the last cone, and when viewed from 3D, the chain of cones looks very similar to the back-chaining approach proposed in~\citet{lozano1984automatic}, and LQR trees proposed in~\citet{tedrake2009lqr}. The main difference to the back-chaining approach is that this game's outcome can have different forms.

In the proposed problem, though the evader has a higher velocity than the pursuers, a fixed set of states it can visit exists. Because the trajectories are linear, it is easy to use simple geometrical analysis to find the boundary condition for winning against the evaders. When the trajectories for the evaders are not straight lines, or when the states $S$ are not stationary, the problem becomes more interesting. 

\textbf{Keep-away game on a graph}: When the game is played on a graph rather than in Euclidean space, we need to find vertices within different distances along the moving paths to find pursuer placements. The number of moving lanes grows on the graph. Again, we can find overlapping vertices among disjoint moving paths and check if formations are feasible and whether the formations can be transformed on the graph. As the moving lanes no longer obey the geometric adjacency relations, the number of pursuers needed to win could increase dramatically. 

\textbf{Non-straight passing lanes}: When the evader can travel along curves, the interception regions change. The proposed algorithm cannot be directly used to compute the winning condition. On the other hand, as the evader eventually has to reach a state, one can place circles of different radii at all states, check for overlapping regions among the circles, and check for transitions among formations. However, the circle-based intersection is only sufficient for the pursuers to win and can be far from optimal. 

\textbf{Moving states}: When the evader travels along straight lines, but the states $s\in S$ can move, the problem becomes more challenging. In this case, the moving lane is no longer just a set of isolated lines but a set of cones that may or may not overlap. The number of pursuers needed to win against the evader increases and may reach $m = n-1$ if the states can move with the same velocities as the pursuers. Again, we can use a procedure similar to the proposed ones to check for winning conditions. In such cases, the angle of the coverage region decreases. The decreased amount is determined by the states' moving velocity, the pursuer velocity, and the evader velocity. 

\textbf{Pursuers with acceleration bound}: When the pursuers can no longer move instantaneously towards an arbitrary direction and have acceleration bounds, the proposed problem becomes even more interesting and more complex simultaneously. The game with acceleration bound is the primary motivation for studying the proposed problem. If we directly study the problem using differential geometry, the game would be very similar to a differential game and complex to solve. We are interested in whether the solution of the proposed game without acceleration bound can be used to solve the case when the acceleration bounds exist. The acceleration bound reduces the coverage region, but the changes can be upper and lower bounded based on the acceleration. Although the boundary condition cannot be directly solved or even described analytically, we believe the solution~to the discrete problem can be constructive to derive the boundary condition and the control synthesis for the pursuers.

\textbf{Pursuit evasion with exits}: In the proposed game, the evaders follow a set of specific paths. What if the winning condition for the evaders is to reach specific locations (exits)? Again, the pursuers must make different formations, but this time, against the exits. We can extend such a game to a non-zero-sum game, where the evaders attempt to make sure at least $k$ of them would escape while the pursuers attempt to catch as many evaders as possible. Is the proposed game related to this pursuit-evasion game with exits? What if the winning condition is that the evader(s) must visit a subset of exits in an arbitrary sequence? We will investigate such variations in future work. 

\section{Conclusion}\label{sec:conclusion}
This work introduces a variation of the pursuit-evasion game as an extension of the keep-away game usually practiced in football. The game shares many interesting attributes with classic lion and man games and homicidal chauffeur games. We show the limit of different winning conditions for the pursuers and evaders. We also show some interesting cases with a small number of visiting states and pursuers. The game, even when simplified, is not trivial. The strategy of the evader and pursuers builds upon solutions of complex optimization but can be analyzed using geometry. We present a solution to the proposed game with no acceleration bound for the pursuers. The approach is built upon computational geometry, and we have presented several example solutions to the game with a different number of states. 

One additional potential future works is to study when the pursuers have acceleration bounds and find solutions by interweaving discrete search and continuous optimization.

%\newpage

\section*{Acknowledgment}

The authors would like to thank some helpful discussions with Professor Chenxi Wu from the Department of Mathematics at the University of Wisconsin-Madison, who was a consulting researcher in the Cognitive Computing Lab at Baidu Research.

\bibliographystyle{plainnat}
\bibliography{refs}

\end{document}